\documentclass{article}

\usepackage[preprint]{neurips_2024}

\usepackage[utf8]{inputenc} 
\usepackage[T1]{fontenc}    
\usepackage{hyperref}       
\usepackage{url}            
\usepackage{booktabs}       
\usepackage{amsfonts}       
\usepackage{nicefrac}       
\usepackage{microtype}      
\usepackage{lipsum}
\usepackage{fancyhdr}       
\usepackage{graphicx}       
\graphicspath{{media/}}     
\usepackage{graphicx,caption, subcaption}
\usepackage{amsfonts, amsmath, amsthm}
\usepackage{algorithm, algpseudocode}
\newtheorem{theorem}{Theorem}[section]
\newtheorem{corollary}{Corollary}[theorem]
\newtheorem{lemma}{Lemma}[section]
\usepackage{natbib}
\usepackage[dvipsnames]{xcolor}
\usepackage{colortbl}

\definecolor{mintbg}{rgb}{.63,.79,.95}
\colorlet{lightmintbg}{mintbg!50}
\colorlet{lightlightgray}{lightgray!40}
\colorlet{lightyellow}{yellow!40}
\colorlet{lightgreen}{LimeGreen!40}

\title{SVDq: 1.25-bit and 410$\times$ Key Cache Compression for LLM Attention
}

\author{
  Yankun Hong, Xing Li, Hui-Ling Zhen, Xianzhi Yu, Wulong Liu, Mingxuan Yuan \\
  Huawei Noah's Ark Lab \\
  \texttt{\{hongyankun, li.xing2, zhenhuiling2, yuxianzhi, liuwulong} \\ \texttt{yuan.mingxuan\}@huawei.com} \\
}

\begin{document}
\maketitle

\begin{abstract}
For the efficient inference of Large Language Models (LLMs), the effective compression of key-value (\(KV\)) cache is essential. Three main types of \(KV\) cache compression techniques, namely sparsity, channel compression, and quantization, have been identified. This study presents SVDq, a Singular Value Decomposition (SVD) - based mixed precision quantization method for \(K\) cache. Initially, \(K\) cache is transformed into “latent channels” using SVD basis representations. Since the values in latent channels decay rapidly and become negligible after only a few latent channels, our method then incorporates importance-aware quantization and compression for latent channels. This enables the effective allocation of higher precision to more significant channels. Theoretically, we prove that SVDq results in quantization errors (\(\times0.1\) or even lower) that are much lower than those of per-channel key quantization in the original space. Our findings based on RULER and LongBench benchmarks demonstrate that SVDq can achieve an equivalent key cache precision \textbf{as low as \(\textbf{1.25}\)-bit}. When combined with key sparsity, it can reach a key compression ratio of up to \textbf{\(\textbf{410}\times\) for attention computation}, all while maintaining comparable model performance. Notably, our method is nearly \textbf{lossless} for LongBench datasets. This indicates that SVDq enables high-precision low-bit quantization, providing a more efficient solution for \(KV\) cache compression in LLMs.
\end{abstract}

\section{Introduction}\label{sec:introduction}

Large Language Models (LLMs) have started a new era of artificial intelligence by demonstrating remarkable capabilities in handling complex tasks \citep{openai_gpt-4_2024, grattafiori_llama_2024, qwen_qwen25_2025, deepseek-ai_deepseek-r1_2025}. Most of these recently developed LLMs are founded upon the attention mechanism based auto-regressive decoder transformers \citep{vaswani_attention_2023}. Consequently, they need to encode past information into intermediate hidden tensors, specifically $KV$ caches, for subsequent and efficient inference.

However, in natural language tasks with large batches or long contexts, $KV$ cache often expands significantly in size, posing a significant challenge to fast inference \citep{pope_efficiently_2022, liu_kivi_2023}. The substantial memory consumption and latency required to save and load $KV$ cache, coupled with the computational demands of attention operations, become critical bottlenecks for LLM inference. 
Considering the rapid advancement of computability and the increasing demand for efficient LLM inference, we recognize the importance of high-ratio $KV$ cache compression (even with a slight concession in computational overhead), enabling the inference of LLMs on devices with limited memory. 

Existing approaches to $KV$ cache compression can be categorized into three main directions: sequence-axis compression, channel-axis compression, and digit-type compression. \emph{(i)} Sequence-axis compression, exemplified by works such as \cite{xiao2023streamingllm, zhang2024h2o, ge2023fastgen, li2024snapkv, tang_quest_2024, ribar_sparq_2024, singhania_loki_2024, yang2025attentionpredictor}, often referred to as sparsity, involves identifying and discarding unimportant tokens for attention computation. 
\emph{(ii)} Channel-axis compression, as demonstrated in, \emph{e.g.}, \cite{xu_think_2024, liu2024deepseekv3, sun_shadowkv_2024}, focuses on the channel dimension compression of $KV$ cache with methods like truncating and low-rank decomposition. Notably, low-rank approximation techniques, as explored in \cite{wang_svd-llm_2024, zhang_lorc_2024}, represent a similar approach of this category. These methods transform $KV$ cache into "\textit{latent channels}" representation based on SVD, and then discard insignificant latent channels. 
\emph{(iii)} Digit-type compression, also known as quantization, aims to reduce the memory footprint by employing lower-precision representations for $KV$ cache \citep{liu_kivi_2023, hooper2024kvquant, yang2024mikv, liu_unlocking_2024, li2025kvtuner}. This typically involves replacing the $32$- or $16$-bit FP numbers with lower precision representations. 
These three compression methods are proposed independently, exploiting different properties of $KV$ cache within LLMs. 

The effectiveness of quantization highly depends on the statistical distribution of the cache values. Large value ranges and outliers can lead to substantial quantization errors. 
In addition, the performance of models degrades significantly below a certain quantization bit width (typically around $4$ to $2$ bits), thus limiting the compression ratio. Similarly, channel compression methods also face challenges in terms of the trade-off between accuracy and compression ratio. While works like \cite{wang_svd-llm_2024, zhang_lorc_2024} have demonstrated $2\times$ compression ratios using SVD-based methods, further compression beyond this point leads to high accuracy loss. Recognizing these limitations, we emphasize the importance of combining these different strategies to further improve the compression ratio. For examples, ThinK \cite{xu_think_2024} highlights the compatibility of its channel truncation method with sparsity techniques; ShadowKV \citep{sun_shadowkv_2024} combines sparsity with SVD low-rank approximation to achieve minor performance degradation while achieving very high compression ratios. 

In this work, we follow the channel-axis compression and quantization strategy. We find that direct truncation of the original channels, as exemplified by ThinK \citep{xu_think_2024}, leads to significant performance degradation when pursuing high compression ratios. To address this challenge, we propose a compression method, SVDq, that integrates the channel truncation and quantization, by utilizing our observed underlying relationship between quantization and SVD-based channel compression.

Specifically, we observe an implication of the Eckart–Young–Mirsky theorem \citep{mirsky_symmetric_1960}: the variances of the values within latent channels obtained through SVD are determined by the corresponding singular values and typically exhibit rapid decay. Recognizing that variances are often proportional to value ranges of latent channels, we can utilize singular values to guide the selection of quantization bit widths to balance accuracy and compression ratios. 

Based on this observation, we propose a novel mixed-precision \textbf{key cache}\footnote{We do not investigate the value cache since it often exhibits weak low-rank property.} quantization method that integrates SVD-based channel compression. This method prioritizes higher bit widths for latent channels associated with larger singular values and progressively decreases precision for channels with smaller singular values. The SVD latent channels offer a significant advantage over simple variance-based descending sorting in the original space, because singular values decay exponentially for most key cache. In consequence, the range at each channel decreases fast, and often becomes insignificant after only a small number of latent channels. Hence, this approach enhances the effectiveness of quantization precision allocation for each latent channel. Furthermore, we emphasize the seamless compatibility of this method with sparsity techniques.

Our key contributions are as follows:
\begin{enumerate}
    \item Proposing a novel method that effectively combines quantization and latent channel compression for $K$ cache, providing the theoretical insights.
    \item Demonstrating the compatibility of this method with sparsity techniques.
    \item Achieving a remarkable level of $K$ cache compression with an equivalent mixed quantization precision as low as $1.25$ bit while maintaining comparable model performance.
\end{enumerate}

\section{Related Works}\label{sec:relatedworks}

\textbf{Sparsity}: With different feature extraction based attention estimation algorithms, methods such as Fastgen \citep{ge2023fastgen}, H2O \citep{zhang2024h2o}, Quest \citep{tang_quest_2024}, SparQ \citep{ribar_sparq_2024}, PQCache \citep{zhang2024pqcache}, ShadowKV \citep{sun_shadowkv_2024}, and AttentionPredictor \citep{yang2025attentionpredictor} selectively retain only the most important tokens in the sequence and effectively prune the others. Loki \cite{singhania_loki_2024} is another sparsity method that uses the SVD approximation to accelerate attention estimation for critical tokens selection.
 
\noindent \textbf{Channel Compression}: These methods, such as ThinK \citep{xu_think_2024}, reduce the dimensionality of $KV$ cache by truncating channels or employing low-rank approximations. Prominent examples include SVD-based approaches like SVD-LLM \citep{wang_svd-llm_2024}, LoRC \citep{zhang_lorc_2024}, Palu \citep{chang_palu_2024}, and Eigen Attention \cite{saxena_eigen_2024}. Notably, techniques like Grouped Query Attention (GQA) \cite{ainslie_gqa_2023}, Multi-head Latent Attention (MLA) \citep{deepseek-ai_deepseek-r1_2025}, and transformations from Multi-Head Attention to GQA \citep{jin_align_2024, chen_optimised_2024} can also be viewed as forms of channel compression, as they effectively reduce the number of attention dimensions.

\noindent \textbf{Quantization}: Methods like KIVI \citep{liu_kivi_2023}, KVQuant \cite{hooper2024kvquant}, AlignedKV \citep{tan_alignedkv_2024}, BitStack \cite{wang_bitstack_2024}, and KVTuner \citep{li2025kvtuner} reduce the memory footprint with low precision $KV$ cache. QServe \cite{lin_qserve_2024} introduces several quantization and system co-design methods to achieve efficient W4A8KV4, where SmoothAttention is utilized to migrate the key quantization difficulty to query.

Some works explore the combination of these approaches. In addition to the mentioned ShadowKV \citep{sun_shadowkv_2024} and ThinK \cite{xu_think_2024}, \cite{liu_unlocking_2024} integrates quantization with matrix decomposition to apply different quantization precision for the two decomposed matrices, and Palu \cite{chang_palu_2024} applies per token quantization to the latent vector of the SVD low-rank approximation.

Importantly, the concept of using SVD for mixed-precision quantization has been explored in other contexts. For instance, Delta-CoMe \citep{ping_delta-come_2024} applies this principle to compress LLM weights, while SVDQuant \citep{li_svdquant_2024} utilizes it for compressing diffusion models. The novelty of this work over the mentioned works lies not only in the application of this principle to $K$ cache compression but also in the \textbf{theoretical foundation} upon which we derive the principle and method, and the \textbf{error analysis} we provide.

\section{SVD and Quantization}\label{sec:SVDquant}

\textbf{Singular Value Decomposition}: Let $\textbf{K} \in \mathbb{R}^{s \times d}$ denotes the $K$ cache matrix for a given head in a transformer layer, where $s$ and $d$ represent the sequence length and hidden embedding (channel) dimension, respectively, with $s \gg d$ typically holding for long context applications. Let $\textbf{K}$ be centered by subtracting its per-channel mean $\bar{\textbf{K}} \in \mathbb{R}^d$, \emph{i.e.}, $\textbf{K} \gets \textbf{K} - \bar{\textbf{K}}$ and maintain the same notation. 

Assuming $\textbf{K}$ is full-rank. Its SVD is given by 
\begin{equation}\label{eq:SVD}
    \textbf{K} = \textbf{U} \cdot \textbf{D} \cdot \textbf{V}^\textrm{H},
\end{equation}
where $\textbf{U} \in \mathbb{R}^{s \times d}$ has orthonormal columns, $\textbf{V} \in \mathbb{R}^{d \times d}$ is orthonormal, satisfying $\textbf{U}^\textrm{H} \cdot \textbf{U} = \textbf{I}_d$ and $\textbf{V}^\textrm{H} \cdot \textbf{V} = \textbf{I}_d$, and $\textbf{D} \in \mathbb{R}^{d \times d}$ is a diagonal matrix containing the singular values in its diagonal with elements arranged in descending order, given by $\textbf{D} = \mathrm{Diag}([\lambda_1, ..., \lambda_{d}])$. 

\noindent \textbf{Quantization} Let $\textbf{\textit{k}}_{\rm min} := (\min{\textbf{K}}_{:1}, ..., \min{\textbf{K}}_{:d})$, \emph{i.e.}, the column-wise minimum vector, and analogously define $\textbf{\textit{k}}_{\rm max}$. The per-channel asymmetrical $b$-bit quantization and dequantization operations are given by: 
\begin{align}
    \mathcal{Q}_b (\textbf{K}) := & \ \left\lfloor\frac{\textbf{K} - \textbf{\textit{k}}_{\rm min}}{(\textbf{\textit{k}}_{\rm max} - \textbf{\textit{k}}_{\rm min}) / (2^b - 1)}\right\rceil, \label{eq:quant} \\
    \mathcal{D}_b (\textbf{K}_b) := & \ \mathcal{Q}_b (\textbf{K}) \times \frac{\textbf{\textit{k}}_{\rm max} - \textbf{\textit{k}}_{\rm min}}{2^b - 1} + \textbf{\textit{k}}_{\rm min}, \label{eq:dequant}
\end{align}
where $\left\lfloor \cdot \right\rceil$ denote the rounding operator. Naturally, $\mathcal{D}_b \circ \mathcal{Q}_b (\textbf{K}) \approx \textbf{K}$. 

For uniformly or normally distributed columns of $\textbf{K}$, the relative quantization errors depend solely on the bit width $b$, independent of the range $\textbf{\textit{k}}_{\rm max} - \textbf{\textit{k}}_{\rm min}$. However, the absolute errors scale with $\textbf{\textit{k}}_{\rm max} - \textbf{\textit{k}}_{\rm min}$, implying that smaller value ranges or variances yield smaller absolute quantization errors.  

\section{Methods}\label{sec:methods}

Although the theory of the proposed SVD-quantization method, discussed in the previous section, is expected to be applicable to a much wider range of applications, this work focuses on KV cache compression in the long context inference scenario. For long context LLMs, $KV$ cache generated in the prefilling stage generally dominates the memory usage. Our method is proposed to address this challenge. 

\subsection{SVD Quantization}\label{subsec:SVDq}

Consider the rows of $\textbf{V}^\mathrm{H}$ in Equation \eqref{eq:SVD} as a basis for the row space of $\textbf{K}$. For the projection $\mathcal{P}_{\textbf{V}_{: j}}$ of the rows of $\textbf{K}$ into the $j$-th basis vector, defined by $\mathcal{P}_{\textbf{V}_{: j}} (\textbf{K}) := \textbf{K} \cdot \textbf{V}_{: j}$, following the Eckart–Young–Mirsky theorem \citep{mirsky_symmetric_1960}, we have: 

\begin{theorem}\label{them:1}
    For the $K$ cache matrix $\boldsymbol{\mathrm{K}}$, the variance of its projection satisfies
    \begin{equation}
        \mathrm{Var} (\mathcal{P}_{\boldsymbol{\mathrm{V}}_{: j}} (\boldsymbol{\mathrm{K}})) = \lambda_{j}^2.
    \end{equation}
\end{theorem}

\begin{corollary}\label{them:2}
    Let $\textbf{k} \in \mathbb{R}^d$ be a $K$ cache vector with $\bar{\boldsymbol{\mathrm{K}}}$ subtracted, \emph{i.e.}, $\textbf{k} \gets \textbf{k} - \bar{\boldsymbol{\mathrm{K}}}$. For any indices $0 < i \leq j <d$, the squared expectations of its projections satisfy:  
    \begin{equation}
        \mathbb{E}((\mathcal{P}_{\boldsymbol{\mathrm{V}}_{: i}} (\textbf{k}))^2) \geq \mathbb{E}((\mathcal{P}_{\boldsymbol{\mathrm{V}}_{: j}} (\textbf{k}))^2).   
    \end{equation}
\end{corollary}

\begin{proof}
    For any $0 < j \leq d$, the projection of $\textbf{K}$ is given by
    \begin{equation*}
        \mathcal{P}_{\textbf{V}_{: j}}(\textbf{K}) = \textbf{K} \cdot \textbf{V}_{: j} = \textbf{U} \cdot \textbf{D} \cdot \textbf{V}^\textrm{H} \cdot \textbf{V}_{: j} = \lambda_j \textbf{U}_{:j}.
    \end{equation*}
    
    Since $\mathbb{E}(\mathcal{P}_{\textbf{V}_{: j}}(\textbf{K})) = \mathcal{P}_{\textbf{V}_{: j}}(\mathbb{E}(\textbf{K})) = 0$, we have
    \begin{equation*}
        \mathrm{Var} (\mathcal{P}_{\textbf{V}_{: j}} (\boldsymbol{\mathrm{K}})) = \mathrm{Var} (\lambda_j \textbf{U}_{:j}) = \lambda_j^2 \mathbb{E}\left(\textbf{U}^\textrm{H}_{j:} \cdot \textbf{U}_{:j}\right) = \lambda_j^2. 
    \end{equation*}
    
    This proves Theorem \ref{them:1}. Corollary \ref{them:2} follows directly from Theorem \ref{them:1} when the given vector $\textbf{\textit{k}}$ follows the distribution of the rows of $\textbf{K}$. 
\end{proof}

\begin{figure}[ht]

\begin{subfigure}{0.48\linewidth}
\includegraphics[width=0.47\linewidth,trim=80 30 80 60,clip]{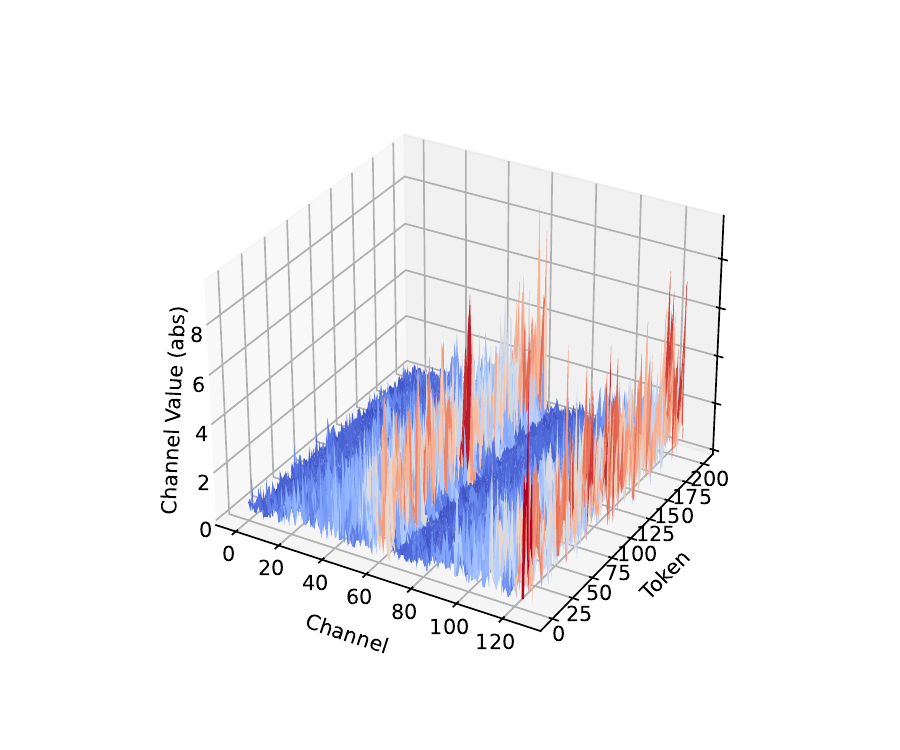}
\includegraphics[width=0.52\linewidth,trim=50 10 30 40,clip]{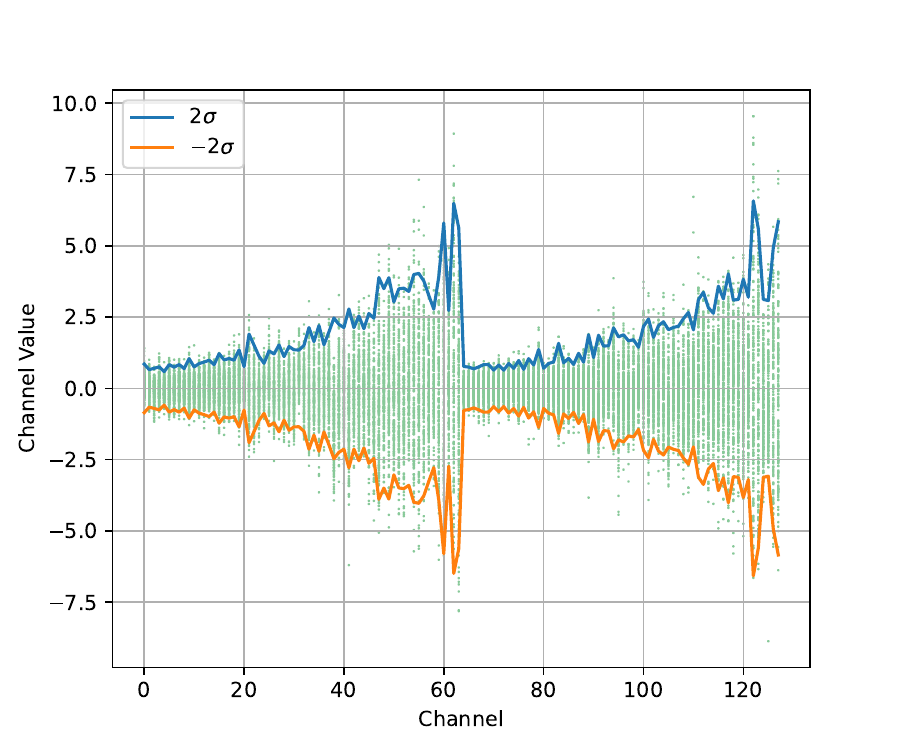}
\caption{Original $\textbf{K}$}
\label{fig:origvar}
\end{subfigure}
\hspace{0.02\linewidth}
\begin{subfigure}{0.48\linewidth}
\includegraphics[width=0.47\linewidth,trim=80 30 80 60,clip]{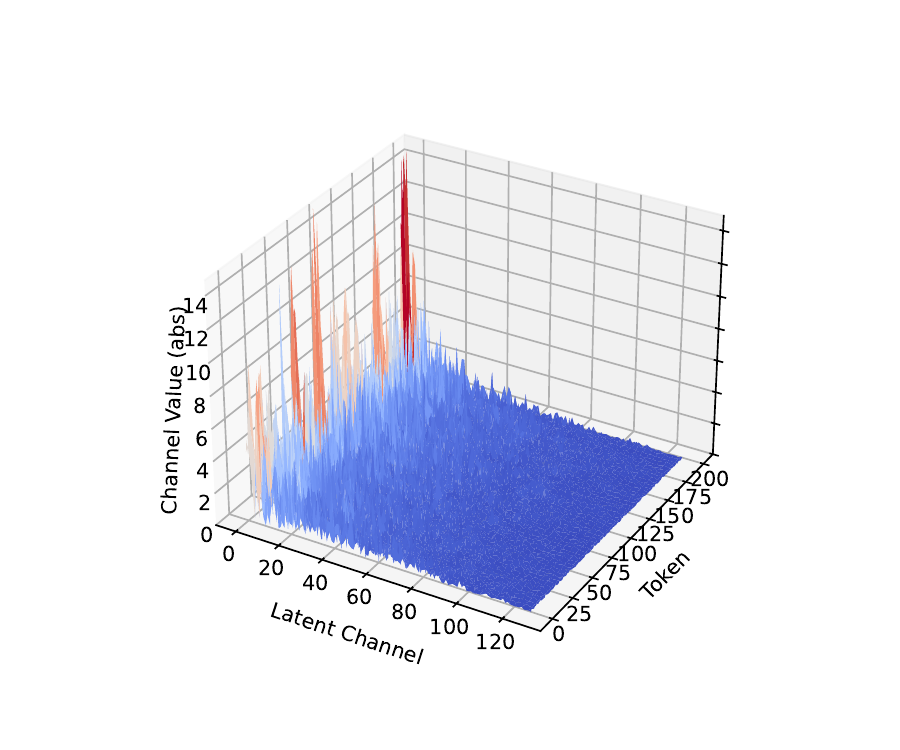}
\includegraphics[width=0.52\linewidth,trim=50 10 30 40,clip]{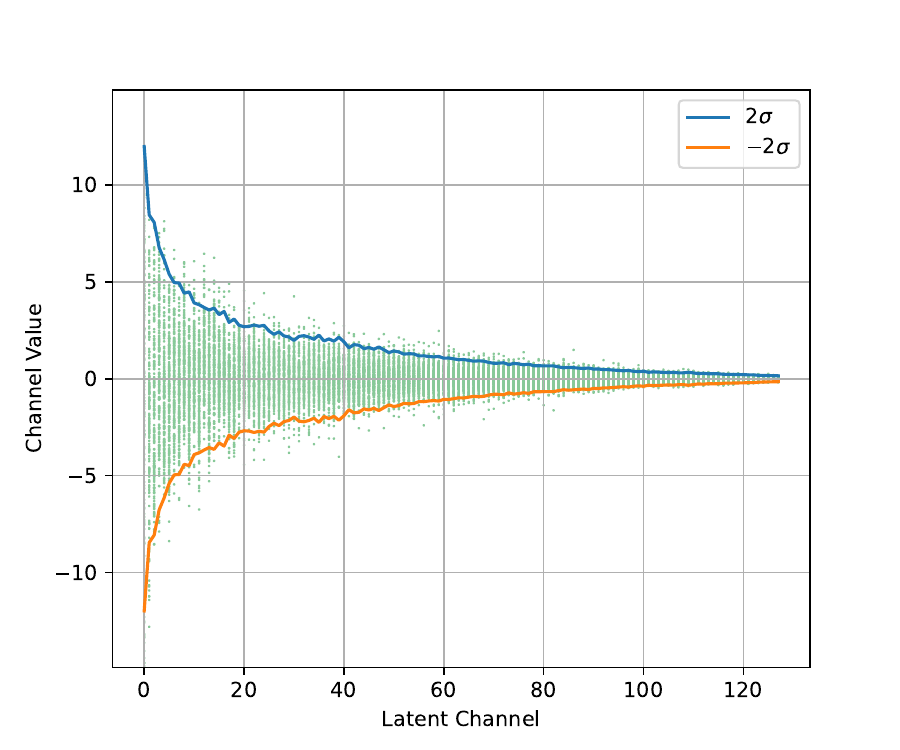}
\caption{Projected $\textbf{K}$, \emph{i.e.}, $\mathcal{P}_{\textbf{V}}(\textbf{K})$}
\label{fig:projvar}
\end{subfigure}

\caption{Distribution of $\textbf{K}$ and its standard deviation}
\label{fig:image2}
\end{figure}

Note that the $\mathcal{P}_{\textbf{V}}(\textbf{K})$ is essentially an alternative representation of $\textbf{K}$ using the singular vector in $\textbf{V}$ as the space bases. We call the columns of $\mathcal{P}_{\textbf{V}}(\textbf{K})$ \textit{latent channels}. Figure \ref{fig:origvar} illustrates the distribution of $\textbf{K}$ in its original space, while Figure \ref{fig:projvar} displays its representation in the SVD space after projection, demonstrating the results of Theorem \ref{them:1} and Corollary \ref{them:2}. 
The singular vector-based projection offers a significant advantage over simple variance-based descending sorting: for most matrices, singular values typically exhibit exponential decay. Consequently, the range of projection values (represented on the $y$-axis in Figure \ref{fig:projvar}) decreases rapidly, becoming relatively insignificant (compared to the value range of the first dimension) after only a small number of latent channels.  

Since $\textbf{K} = \mathcal{P}_{\textbf{V}} (\textbf{K}) \cdot \textbf{V}^\textrm{H}$ where $\mathcal{P}_{\textbf{V}}(\textbf{K}) := \textbf{K} \cdot \textbf{V}$, and all basis vectors in $\textbf{V}$ are unit-normalized, the absolute error in approximating to $\mathcal{P}_{\textbf{V}}$ represents both absolute and relative errors in approximating $\textbf{K}$. Theorem \ref{them:1}, Corollary \ref{them:2}, and Figure \ref{fig:projvar} demonstrate that both value range and variance decay rapidly along the latent channels. This property motivates our efficient mixed-precision quantization method, SVDq, to approximate $\textbf{K}$ via $\mathcal{P}_{\textbf{V}}(\textbf{K})$:

\noindent \emph{\textbf{(1)}} Use high precision quantization for initial latent channels;

\noindent \emph{\textbf{(2)}} Progressively decrease the precision for subsequent latent channels;

\noindent \emph{\textbf{(3)}} Truncate the remaining latent channels with negligible value ranges or singular values.

\begin{figure*}
    \centering
    \includegraphics[width=\linewidth]{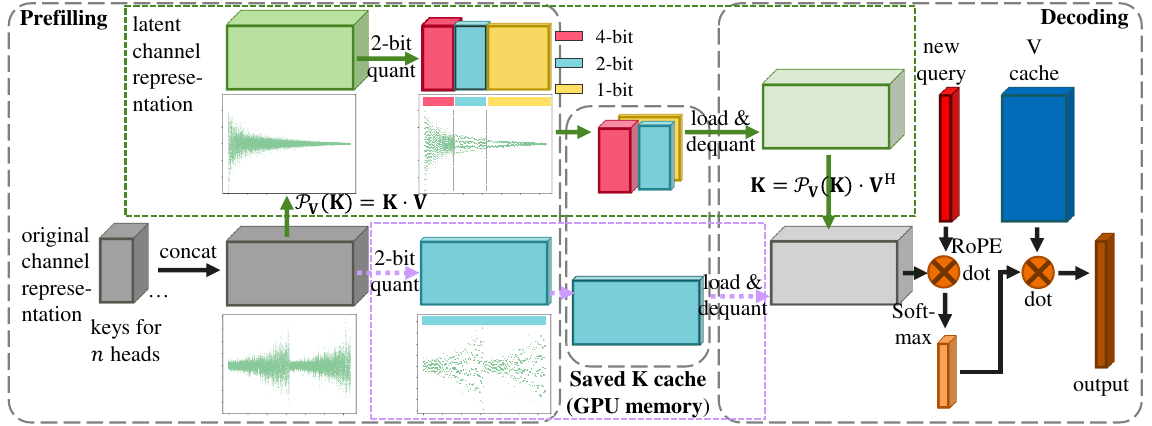}
    \caption{Diagram of SVDq method (path inside the box in green) versus direct per-channel quantization (dash path inside the box in violet).}
    \label{fig:diagram}
\end{figure*}

\subsection{Algorithm}\label{subsec:algorithm}
In our SVDq method, we first apply SVD to the prefilling $K$ cache, obtaining the projection operator $\mathcal{P}_{\textbf{V}}(\cdot)$ using the right SVD matrix $\textbf{V}$. Next, we determine a precision schedule for the quantization on each latent channel based on the singular values $[\lambda_1, ..., \lambda_{d}]$. Specifically, a latent channel associated with a large singular value $\lambda$ is assigned a high quantization bit width $b$, and channels with small $\lambda$ are assigned low $b$ or even be truncated with notation $b = 0$. This yields a schedule vector $\textit{\textbf{b}}$, and the equivalent mixed bit width of this quantization schedule for the $\textbf{K}$ cache is given by
\begin{equation}
    \bar{b} = \frac{1}{d} \sum_{i=1}^d \textit{\textbf{b}}_i. 
\end{equation}
Sequently, we use $\mathcal{Q}_\textit{\textbf{b}}$ in \eqref{eq:quant} to quantize $\mathcal{P}_{\textbf{V}}(\textbf{K})$. The low-bit quantized $\mathcal{P}_{\textbf{V}}(\textbf{K})$ is then saved as the cache. In the decoding process, we dequantize the cache, reconstruct $\textbf{K}$ in its original representation using $\textbf{K} = \mathcal{P}_{\textbf{V}}(\textbf{K}) \cdot \textbf{V}^\mathrm{H}$, and then proceed with the attention computation. We summarize the algorithm using pseudo-code in Algorithm \ref{alg:SVDq} and an abstracted diagram in Figure \ref{fig:diagram}.

\begin{algorithm}[ht]
    \caption{SVD-quantization algorithm for $\textbf{K}$}\label{alg:SVDq}
    \begin{algorithmic}
        \Require $K$ cache matrix $\textbf{K}$ of $L$ layers
        \Ensure $\hat{\textbf{K}} \approx \textbf{K}$
        \For {$l \gets 1$ to $N$}
            \State Load $\textbf{K}$ cache matrix for $l$-th layer
            \State $\textbf{K} = \textbf{K} - \bar{\textbf{K}}$
            \State $\textbf{U}, \textbf{D}, \textbf{V} \gets \mathrm{SVD} (\textbf{K})$ 
            \State $\mathcal{P}_{\textbf{V}}(\textbf{K}) = \textbf{U} \cdot \textbf{D}$
            \State Set quant schedule $\textit{\textbf{b}}$
            \State Save $\bar{\textbf{K}}$, $\textbf{V}$, $\mathcal{Q}_\textit{\textbf{b}} \circ \mathcal{P}_{\textbf{V}}(\textbf{K})$, function $\mathcal{D}_\textit{\textbf{b}}$
        \EndFor
        \State $\hat{\textbf{K}} = \mathcal{D}_\textit{\textbf{b}} \left( \mathcal{Q}_\textit{\textbf{b}} \circ \mathcal{P}_{\textbf{V}}(\textbf{K}) \right) \cdot \textbf{V}^\mathrm{H} + \bar{\textbf{K}}$
    \end{algorithmic}
\end{algorithm}

In this algorithm, the quantities to be saved include the quantized $\mathcal{P}_{\textbf{V}}(\textbf{K}) \in \mathbb{R}^{s \times d}$ (represented using $\bar{b}$-bit), the right SVD matrix $\textbf{V} \in \mathbb{R}^{d \times d}$, the average of $\textbf{K}$ denoted by $\bar{\textbf{K}} \in \mathbb{R}^d$, and the dequant function $\mathcal{D}_\textit{\textbf{b}}$, which relies on the bit schedule $\textit{\textbf{b}} \in \mathbb{R}^d$ and the range of $\mathcal{P}_{\textbf{V}}(\textbf{K})$, given by $\textit{\textbf{p}}_{\rm min}, \textit{\textbf{p}}_{\rm max} \in \mathbb{R}^d$. In long context applications, $d \ll s$, the requirement of memory space for terms that depend solely on $d$, \emph{e.g.}, the space for $\textbf{V}$ and $\bar{\textbf{K}}$, is negligible.
Hence, the compression rate compared with the original $16$-bit $\textbf{K}  \in \mathbb{R}^{s \times d}$ is approximately $16 / \bar{b}$. 

In this work, we concatenate the $\textbf{K}$ matrices of all heads within the same layer, resulting in a larger $\textbf{K}$ matrix with the embedding dimension $d$ being the sum of the embedding dimensions of all heads.
To improve efficiency, for the bit schedule setting $\textit{\textbf{b}}$, we divide the $d$ latent channels of $\mathcal{P}_{\textbf{V}}(\textbf{K})$ into $8$ equal-sized group, each comprising $\frac{d}{8}$ dimensions. The channels within each group share the same quantization bit width. Thus, $\textit{\textbf{b}}$ is determined by an 8-dimensional vector $(b_1, b_2, ..., b_8)$ of integer. For example, a schedule of $(8,4,2,1,1,0,0,0)$ has an equivalent mixed bit width $\bar{b} = 2$ and hence a compression ratio $8$. For a model with $d = 1024$, this schedule implies: 
\begin{itemize}
    \item $8$-bit quantization for the first $128$ latent channels,
    \item $4$-bit for the next $128$ channels,
    \item $2$-bit for the next $128$ channels,
    \item $1$-bit for the next $256$ channels,
    \item truncation for the remaining $384$ channels. 
\end{itemize}
\subsection{Theoretical Error Analysis}\label{subsec:erranalysis}

We begin by presenting a lemma for later analysis.

\begin{lemma}\label{lem:4.1}
    If data $X$ are distributed uniformly within their value range $r$, then the expectation of the square absolute error, $\varepsilon$, of an asymmetrical $b$-bit quantization applied to $X$ is equal to the variance of a uniform distribution with a range of $\frac{r}{2^b}$, that is
    \begin{equation*}
        \mathbb{E} (\varepsilon^2) = \frac{1}{12} \frac{r^2}{2^{2b}}.
    \end{equation*}
\end{lemma}

Let $\textbf{K}$ be centered by subtracting the key's per-channel mean $\bar{\textbf{K}} \in \mathbb{R}^d$, and let $\mathcal{P}_\textbf{V} (\textbf{K})$ be its latent channel representation. The Frobenius norm is invariant under this transformation, as 
\begin{equation*}
    \|\mathcal{P}_\textbf{V} (\textbf{K})\|_\mathrm{F}^2 = \sum_{i=1}^{s} \mathcal{P}_\textbf{V} (\textbf{K}_{i:}) \cdot \mathcal{P}_\textbf{V} (\textbf{K}_{i:})^\mathrm{H} = \sum_{i=1}^{s} \textbf{K}_{i:} \cdot \textbf{K}_{i:}^\mathrm{H} = \|\textbf{K}\|_\mathrm{F}^2.
\end{equation*}

Let $[\sigma_1^2, ..., \sigma_d^2]$ and $[\lambda_1^2, ..., \lambda_d^2]$ denote the variance vectors of the channels for the original and latent channel representations of $\textbf{K}$, respectively. Thus, 
\begin{equation}\label{eq:vareq}
    \sum_{j=1}^{d} \sigma_j^2 = \frac{1}{s} \|\textbf{K}\|_\mathrm{F}^2 = \frac{1}{s} \|\mathcal{P}_\textbf{V} (\textbf{K})\|_\mathrm{F}^2 = \sum_{j=1}^{d} \lambda_j^2.
\end{equation}

We further assume that the key cache distributions in each original channel and latent channel follow uniform distributions. Then, according to Lemma \ref{lem:4.1}, the value ranges of the $j$-th original channel and $j$-th latent channel are $r_j = 2\sqrt{3} \sigma_j$ and $\hat{r}_j = 2\sqrt{3} \lambda_j$, respectively.

\noindent \textbf{Error analysis for direct quantization} Figure \ref{fig:origvar} shows that the variances in the original channels often exhibit similar orders of magnitude. We therefore assume that they are approximately identical, with $\sigma_j^2 = \frac{1}{ds} \|\textbf{K}\|_\mathrm{F}^2$ and $r_j^2 = \frac{12}{ds} \|\textbf{K}\|_\mathrm{F}^2$. Applying a per-channel, direct $b$-bit quantization to $\textbf{K}$, and following Lemma \ref{lem:4.1} and the above analysis, results in a quantization error $\varepsilon_b$ with the expected value:
\begin{equation}\label{eq:qerr1}
    \mathbb{E} (\varepsilon_b^2) = \frac{1}{12} \frac{1}{2^{2b}} \frac{12}{ds} \|\textbf{K}\|_\mathrm{F}^2 = \frac{1}{2^{2b}} \frac{\|\textbf{K}\|_\mathrm{F}^2}{ds}. 
\end{equation}

\noindent \textbf{Error analysis for SVDq} The singular values of a matrix often exhibit exponential decay. We model the variance vector for $\textbf{K}$'s latent channel representation as 
\begin{equation}
    \lambda_j = c e^{-\rho j} = \lambda_i e^{-\rho (j-i)}, 
\end{equation}
for any $1 \leq i < j \leq d$, where $c > 0$ and $\rho > 0$ are parameters. 

Using this model and \eqref{eq:vareq}, we immediately obtain
\begin{equation*}
    c^2 = \frac{e^{2\rho} - 1}{1 - e^{-2 \rho d}} \frac{\|\textbf{K}\|_\mathrm{F}^2}{s} \approx \frac{e^{2\rho} - 1}{s} \|\textbf{K}\|_\mathrm{F}^2, 
\end{equation*}
as well as the square of the value range of each latent channel 
\begin{equation}
    \hat{r}_j^2 = 12 \frac{(e^{2\rho} - 1) e^{-2\rho j}}{s} \|\textbf{K}\|_\mathrm{F}^2 = 12 (e^{2\rho} - 1) e^{-2\rho j} 2^{2b} d \, \mathbb{E}(\varepsilon_b^2). 
\end{equation}

For further analysis, we set the bit schedule as a simple decreased arithmetic progression\footnote{This setting is introduced only for the sake of clear theoretical error analysis, as it yields concise error expressions. It is not a realistic schedule because it may contain no integer bit widths. A similar analysis can be applied to other schedules, although the derivations may become more complex.}: $b_i = (8-i) \frac{2b}{7}$, resulting in $\bar{b} = \sum_{i=1}^8 b_i = b$, and compare SVDq with this schedule to a direct $b$-bit quantization. Using Lemma \ref{lem:4.1}, for the $i$-th part with $\frac{d}{8}$ latent channels with quantization bit width of $b_i$, the expectation of the square quantization error, $\hat{\varepsilon}_i$, is
\begin{equation*}
    \mathbb{E}(\hat{\varepsilon}_i^2) = \frac{8}{d} \sum_{j=d(i-1)/8+1}^{di/8} \frac{1}{12} \frac{\hat{r}_j^2}{2^{2b_i}} = 8 \frac{e^{2\rho} - 1}{2^{2 (b_i - b)}} \mathbb{E}(\varepsilon_b^2) \sum_{j=d(i-1)/8+1}^{di/8} e^{-2\rho j} \approx 8 \frac{e^{-d\rho (i-1)/4}}{e^{(b_i - b) \ln{4}}} \mathbb{E}(\varepsilon_b^2).
\end{equation*}

Denoting $\hat{b}_i := b_1 - b_i = (i-1)\frac{2b}{7}$ and $\alpha := \frac{d\rho}{4} - \frac{2b}{7} \ln{4}$, the error for SVDq, $\hat{\varepsilon}_{b}$, satisfies 
\begin{equation*}
\begin{split}
    \mathbb{E}(\hat{\varepsilon}_{b}^2) &= \frac{1}{8} \sum_{i=1}^8 \mathbb{E}(\hat{\varepsilon}_i^2) = \mathbb{E}(\varepsilon_b^2) \sum_{i=1}^8 \frac{e^{-d\rho (i-1)/4}}{e^{(b_i-b) \ln{4}}} = \frac{\mathbb{E}(\varepsilon_b^2)}{4^{b_1-b}} \sum_{i=1}^8 e^{-d\rho (i-1)/4 +\hat{b}_i \ln{4}} \\
    &= \frac{\mathbb{E}(\varepsilon_b^2)}{4^{b_1-b}} \sum_{i=1}^8 e^{-\alpha (i-1)} = \frac{1}{4^{b_1-b}} \frac{1 - e^{-8\alpha}}{1-e^{-\alpha}} \mathbb{E}(\varepsilon_b^2). 
\end{split}
\end{equation*}

For LLMs like Llama-3.1-8B, $d=1024$, the decay rate $\rho$ is often on the order of approximately $0.1$, while we typically consider quantization bit widths at the levels $b=2$ or $4$. Consequently, we often have $\rho \gg \frac{8b}{7d} \ln{4}$, resulting in $\alpha \gg 0$. Under these conditions, typically $\left(\frac{\mathbb{E}(\hat{\varepsilon}_{b}^2)}{\mathbb{E}(\varepsilon_b^2)}\right)^{\frac{1}{2}} \approx 2^{b-b_1} < 0.1$, the expectation quantization error of SVDq is much smaller than the direct per-channel quantization error. This result theoretically proves the efficiency of mixed-precision quantization in the latent channel representation guided by SVD.

\section{Experiments}\label{sec:experiments}

\begin{table}
\centering
\resizebox{0.6\linewidth}{!}{
\begin{tabular}{crrrr}
\hline
\textbf{Model} & $d_h$ & $n$ & $d$ & \textbf{part dim} $\frac{d}{8}$ \\
\hline
Llama-3.1-8B-Instruct & $128$ & $8$ & $1024$ & $128$ \\
Qwen2.5-14B-Instruct & $128$ & $8$ & $1024$ & $128$ \\
Qwen2.5-7B-Instruct & $128$ & $4$ & $512$ & $64$ \\
Qwen2.5-3B-Instruct & $128$ & $2$ & $256$ & $32$ \\
\hline
\end{tabular}
}
\caption{\label{tab:config}
Configuration of $K$ cache for four models. 
}
\end{table}

In this section, we apply our method in different model settings to showcase its efficiency in $K$ cache compression. 

We focus on long context applications using four large language models: Llama-3.1-8B-Instruct \citep{grattafiori_llama_2024}, Qwen2.5-14B-Instruct, Qwen2.5-7B-Instruct, and Qwen2.5-3B-Instruct \citep{qwen_qwen25_2025}. The numerical experiments are based on the RULER benchmarks \citep{hsieh_ruler_2024} and LongBench benchmarks \cite{bai2023longbench}. We omit the scores for RULER NIAH Single tests because in our tests, almost all methods achieved perfect scores ($100$) on these tests, indicating that they do not pose a sufficient challenge. We present the results of RULER in Sections \ref{subsec:ResSVDq}-\ref{subsec:ResSVDqspQV} and refer the readers to Section \ref{subsec:res_longbench} for the results of LongBench. 

The configuration settings for the $K$ cache of the four models are listed in Table \ref{tab:config}. The long context prompt length is set to $64$K, satisfying $s = 64 \times 1024 \gg d$.

\subsection{Results of SVDq}\label{subsec:ResSVDq}

In our first experiment, we implement the SVD quantization method directly in $K$ cache compression and summarize the results in Table \ref{tab:SVDresult}. Detailed experiment settings and descriptions are provided in the Appendix \ref{app:5.1}. 

The results demonstrate that the proposed SVDq method generally results in lower performance degradation compared to direct quantization and channel compression across almost all tests. On average, the SVDq method achieves higher scores despite having a lower equivalent mixed quantization bit width. This clearly showcases the significant advantage of truncating and quantizing the SVD latent channels over operating directly on the original channels.

Please note that in our tests, both direct $2$-bit quantization of the original $\textbf{K}$ and equivalent $2$-bit ThinK that retains $\frac{1}{2}$ original channels and combines $4$-bit quantization result in much more significant performance degradation. Therefore, we opted to compare our SVDq method in $2$- and $3$-bit setting with direct $3$-bit quantization and equivalent $3$-bit ThinK for a more meaningful evaluation.

\begin{table*}
\centering
\resizebox{\textwidth}{!}{
\begin{tabular}{crr|rrrrrrrr|r}
\toprule
Method & bit & CR & N-MK1 & N-MK2 & N-MQ & N-MV & VT & FWE & QA-1 & QA-2 & Average \\ \hline
\multicolumn{11}{c}{\textbf{Llama-3.1-8B-Instruct}} \\ 
\midrule
\rowcolor{lightlightgray} Default & $16$ & $1.0$ & $99.0$ & $97.9$ & $98.7$ & $98.2$ & $97.5$ & $85.4$ & $82.3$ & $60.4$ & $90.0$ \\
Per-channel Quant & $3$ & $5.3$ & $97.9$ & $70.8$ & $94.0$ & $91.1$ & $86.0$ & $84.7$ & $67.7$ & $46.9$ & $79.9$ \\
ThinK & $3$ & $5.3$ & $94.8$ & $66.7$ & $87.5$ & $80.7$ & $66.2$ & $90.3$ & $75.0$ & $55.2$ & $77.2$ \\
\rowcolor{lightmintbg} SVDq (ours) & $3$ & $5.3$ & $100.0$ & $96.9$ & $99.2$ & $95.3$ & $97.3$ & $86.1$ & $85.4$ & $57.3$ & $\textbf{89.7}$ \\
\rowcolor{lightmintbg} SVDq (ours) & $2$ & $\textbf{8.0}$ & $99.0$ & $94.8$ & $96.1$ & $92.7$ & $99.0$ & $84.4$ & $75.0$ & $47.9$ & $\underline{86.1}$ \\
\midrule
\multicolumn{11}{c}{\textbf{Qwen2.5-14B-Instruct}} \\ 
\midrule
\rowcolor{lightlightgray} Default & $16$ & $1.0$ & $91.7$ & $41.7$ & $98.2$ & $90.1$ & $96.9$ & $93.8$ & $53.1$ & $49.0$ & $76.8$ \\
Per-channel Quant & $3$ & $5.3$ & $58.3$ & $8.33$ & $76.3$ & $79.7$ & $87.9$ & $94.8$ & $26.0$ & $38.5$ & $58.7$ \\
ThinK & $3$ & $5.3$ & $75.0$ & $25.0$ & $85.7$ & $87.2$ & $92.7$ & $89.9$ & $35.4$ & $30.2$ & $65.1$ \\
\rowcolor{lightmintbg} SVDq (ours) & $3$ & $5.3$ & $85.4$ & $42.7$ & $96.6$ & $85.4$ & $97.5$ & $94.1$ & $55.2$ & $46.9$ & $\textbf{75.6}$ \\
\rowcolor{lightmintbg} SVDq (ours) & $2$ & $\textbf{8.0}$ & $65.6$ & $32.3$ & $90.9$ & $91.1$ & $97.9$ & $94.4$ & $59.4$ & $47.9$ & $\underline{72.5}$ \\
\midrule
\multicolumn{11}{c}{\textbf{Qwen2.5-7B-Instruct}} \\ 
\midrule
\rowcolor{lightlightgray} Default & $16$ & $1.0$ & $86.5$ & $26.0$ & $95.8$ & $87.5$ & $85.8$ & $83.0$ & $61.5$ & $38.5$ & $70.6$ \\
Per-channel Quant & $3$ & $5.3$ & $37.5$ & $3.1$ & $46.9$ & $47.7$ & $63.5$ & $77.1$ & $18.8$ & $25.0$ & $39.9$ \\
ThinK & $3$ & $5.3$ & $60.4$ & $8.3$ & $66.9$ & $71.1$ & $63.7$ & $76.7$ & $40.6$ & $35.4$ & $52.9$ \\
\rowcolor{lightmintbg} SVDq (ours) & $3$ & $5.3$ & $88.5$ & $29.2$ & $92.7$ & $80.2$ & $84.0$ & $87.8$ & $54.2$ & $40.6$ & $\textbf{69.7}$ \\
\rowcolor{lightmintbg} SVDq (ours) & $2$ & $\textbf{8.0}$ & $78.1$ & $36.5$ & $81.8$ & $82.6$ & $79.4$ & $71.5$ & $39.6$ & $32.3$ & $\underline{62.7}$ \\
\midrule
\multicolumn{11}{c}{\textbf{Qwen2.5-3B-Instruct}} \\ 
\midrule
\rowcolor{lightlightgray} Default & $16$ & $1.0$ & $78.1$ & $27.1$ & $89.8$ & $88.8$ & $81.0$ & $72.2$ & $41.7$ & $30.2$ & $63.6$ \\
Per-channel Quant & $3$ & $5.3$ & $27.1$ & $3.1$ & $23.2$ & $25.8$ & $61.7$ & $63.2$ & $14.6$ & $24.0$ & $30.3$ \\
ThinK & $3$ & $5.3$ & $38.5$ & $7.3$ & $49.5$ & $47.9$ & $64.8$ & $66.3$ & $26.0$ & $25.0$ & $40.7$ \\
\rowcolor{lightmintbg} SVDq (ours) & $3$ & $5.3$ & $66.7$ & $15.6$ & $79.7$ & $75.3$ & $74.2$ & $66.7$ & $24.0$ & $27.1$ & $\textbf{53.6}$ \\
\rowcolor{lightmintbg} SVDq (ours) & $2$ & $\textbf{8.0}$ & $52.1$ & $16.7$ & $57.8$ & $56.0$ & $69.8$ & $58.7$ & $19.8$ & $27.1$ & $\underline{44.7}$ \\
\bottomrule
\end{tabular}
}
\caption{\label{tab:SVDresult}
Performance of our method ("SVDq") for key compression in different models on the RULER benchmark evaluated at a context length of 64K. The bit schedules for SVDq are $\textbf{\textit{b}} = (8, 4, 4, 4, 2, 2, 0, 0), (8, 4, 4, 0, 0, 0, 0, 0)$, resulting in $\bar{b} = 3, 2$, respectively. The third column ("CR") is refer to as compression ratio given by $16 / \bar{b}$. The second row ("Per-channel Quant") refers to applying direct per-channel quantization to the original $\textbf{K}$. The thrid row ("ThinK") refers to applying the ThinK method \citep{xu_think_2024} with $\frac{3}{4}$ compression ratio to the original $\textbf{K}$, combining $4$-bit quantization. Our method outperforms direct quantization and ThinK with quantization despite having a lower (mixed) bit width ($2$ bits versus $3$ bits). The value cache is retained in BF16 type without any processes. Detailed settings are found in the Appendix \ref{app:5.1}. 
}
\end{table*}

\subsection{Results of SVDq with Sparsity}\label{subsec:ResSVDqsp}

Although SVDq can improve model performance while using small bit quantizations, significant performance loss can still occur when the bit width is extremely low, such as $\bar{b} = 2$. Hence, we combine our SVDq method with a sparsity technique to investigate its compatibility with other techniques and explore potential performance improvements.

We adopt the sparsity strategy proposed in the ShadowKV method \citep{sun_shadowkv_2024}. Table \ref{tab:SVDspresult} presents the results for sparsity ("ShadowKV SparsityOnly") and ShadowKV ("ShadowKV") as baselines. Please see a brief introduction of the ShadowKV and the description of these baseline settings in the Appendix \ref{app:5.2}. For the SVDq method, we investigate different quantization bit schedules with varying equivalent mixed bit widths: $\bar{b} = 2.25, 1.75$, and $1.25$. Detailed schedules are provided in Table \ref{tab:bitschedule} in the Appendix \ref{app:5.2}. We apply the SVDq in conjunction with the sparsity method from ShadowKV. The scores are presented in the yellow rows in Table \ref{tab:SVDspresult}.

Our observations reveal that, when combined with sparsity, our SVDq compression method does not result in significant performance degradation, even with extremely low quantization bit widths such as $\bar{b} = 1.25$. Decreasing the bit width from $\bar{b} = 2.25$ to $\bar{b} = 1.75$ has a negligible impact on the score. Further decreasing $\bar{b}$ to $1.25$ results in a slight performance loss, although it remains relatively insignificant. Notably, our quantization method, even with $\bar{b} = 1.25$, outperforms the low-rank approximation used in ShadowKV, demonstrating the ineffectiveness of directly truncating SVD ranks. Taking into account the sparsity compression ratio of $32\times$, SVDq contributes an additional ratio of up to $12.8\times$, resulting in a total compression ratio of $400\times$. 

Notably, by comparing Tables \ref{tab:SVDresult} and \ref{tab:SVDspresult}, the introduction of sparsity does not result in performance degradation; it can even improve the performance of models that solely use SVDq or low-rank compression. We observe that with sparsity, the model can withstand higher compression ratios. This may be attributed to the fact that quantization and low-rank approximation introduce errors across all tokens, potentially leading to significant error accumulation in the full attention mechanism. However, sparsity discards unimportant tokens, which can help to mitigate the error from these tokens and improve overall performance.

\begin{table*}
\centering
\resizebox{\textwidth}{!}{
\begin{tabular}{crr|rrrrrrrr|r}
\toprule
Method & bit & CR & N-MK1 & N-MK2 & N-MQ & N-MV & VT & FWE & QA-1 & QA-2 & Average \\ \hline
\multicolumn{11}{c}{\textbf{Llama-3.1-8B-Instruct}} \\ 
\midrule
\rowcolor{lightlightgray} Default & $16$ & $1$ & $99.0$ & $97.9$ & $98.7$ & $98.2$ & $97.5$ & $85.4$ & $82.3$ & $60.4$ & $90.0$ \\
ShadowKV SparsityOnly & $16$ & $32$ & $100.0$ & $97.9$ & $99.0$ & $94.5$ & $89.6$ & $74.0$ & $82.3$ & $61.5$ & $\textbf{87.3}$ \\
ShadowKV & $2.5$ & $205$ & $99.0$ & $97.9$ & $99.0$ & $96.1$ & $85.6$ & $75.0$ & $82.3$ & $59.4$  & $86.8$ \\
\rowcolor{lightyellow} SVDq+Sparsity & $2.25$ & $227$ & $100.0$ & $97.9$ & $98.4$ & $95.3$ & $89.6$ & $74.0$ & $81.2$ & $60.4$ & $\underline{87.1}$ \\
\rowcolor{lightyellow} SVDq+Sparsity & $1.75$ & $291$ & $100.0$ & $97.9$ & $98.7$ & $94.5$ & $88.7$ & $74.7$ & $83.3$ & $60.4$ & $\textbf{87.3}$ \\
\rowcolor{lightyellow} SVDq+Sparsity & $1.25$ & $\textbf{410}$ & $99.0$ & $96.6$ & $99.2$ & $93.2$ & $87.3$ & $74.3$ & $83.3$ & $60.4$ & $86.7$ \\
\rowcolor{lightgreen} SVDq+Sparsity+V4 & $2.25$ & $227$ & $100.0$ & $97.9$ & $98.4$ & $95.3$ & $88.3$ & $75.0$ & $81.2$ & $60.4$ & $\underline{87.1}$ \\
\rowcolor{lightgreen} SVDq+Sparsity+V4 & $1.75$ & $291$ & $100.0$ & $96.9$ & $99.0$ & $94.5$ & $87.7$ & $75.7$ & $832.3$ & $60.4$ & $\underline{87.1}$ \\
\rowcolor{lightgreen} SVDq+Sparsity+V4 & $1.25$ & $\textbf{410}$ & $99.0$ & $96.9$ & $99.2$ & $93.0$ & $86.2$ & $73.3$ & $83.3$ & $60.4$ & $86.4$ \\
\midrule
\multicolumn{11}{c}{\textbf{Qwen2.5-14B-Instruct}} \\ 
\hline
\rowcolor{lightlightgray} Default & $16$ & $1$ & $91.7$ & $41.7$ & $98.2$ & $90.1$ & $96.9$ & $93.8$ & $53.1$ & $49.0$ & $76.8$ \\
ShadowKV SparsityOnly & $16$ & $32$ & $90.6$ & $38.5$ & $96.1$ & $87.0$ & $95.2$ & $86.8$ & $55.2$ & $44.8$ & $\textbf{74.3}$ \\
ShadowKV & $2.5$ & $205$ & $88.5$ & $38.5$ & $94.0$ & $78.6$ & $93.7$ & $88.2$ & $52.1$ & $46.9$ & $72.6$ \\
\rowcolor{lightyellow} SVDq+Sparsity & $2.25$ & $227$ & $88.5$ & $36.5$ & $96.6$ & $86.7$ & $96.7$ & $86.5$ & $56.2$ & $44.8$ & $\underline{74.1}$ \\
\rowcolor{lightyellow} SVDq+Sparsity & $1.75$ & $291$ & $87.5$ & $38.5$ & $94.8$ & $83.1$ & $95.2$ & $87.5$ & $54.2$ & $43.8$ & $73.1$ \\
\rowcolor{lightyellow} SVDq+Sparsity & $1.25$ & $\textbf{410}$ & $89.6$ & $34.4$ & $94.0$ & $85.4$ & $96.5$ & $88.2$ & $54.2$ & $42.7$ & $73.1$ \\
\rowcolor{lightgreen} SVDq+Sparsity+V4 & $2.25$ & $227$ & $88.5$ & $34.4$ & $95.3$ & $85.7$ & $96.0$ & $85.4$ & $57.3$ & $42.7$ & ${73.2}$ \\
\rowcolor{lightgreen} SVDq+Sparsity+V4 & $1.75$ & $291$ & $88.5$ & $36.5$ & $96.1$ & $82.8$ & $95.4$ & $86.8$ & $54.2$ & $42.7$ & $73.1$ \\
\rowcolor{lightgreen} SVDq+Sparsity+V4 & $1.25$ & $\textbf{410}$ & $87.5$ & $35.4$ & $94.8$ & $84.1$ & $96.0$ & $87.5$ & $55.2$ & $43.8$ & $73.0$ \\
\midrule
\multicolumn{11}{c}{\textbf{Qwen2.5-7B-Instruct}} \\ 
\midrule
\rowcolor{lightlightgray} Default & $16$ & $1$ & $86.5$ & $26.0$ & $95.8$ & $87.5$ & $85.8$ & $83.0$ & $61.5$ & $38.5$ & $70.6$ \\
ShadowKV SparsityOnly & $16$ & $32$ & $85.4$ & $19.8$ & $93.5$ & $87.2$ & $86.9$ & $70.8$ & $65.6$ & $35.4$ & $68.1$ \\
ShadowKV & $2.5$ & $205$ & $86.5$ & $17.7$ & $89.8$ & $75.8$ & $71.2$ & $62.8$ & $67.7$ & $37.5$ & $63.6$ \\
\rowcolor{lightyellow} SVDq+Sparsity & $2.25$ & $227$ & $89.6$ & $19.8$ & $94.3$ & $89.6$ & $85.6$ & $69.1$ & $67.7$ & $38.5$ & $\textbf{69.3}$ \\
\rowcolor{lightyellow} SVDq+Sparsity & $1.75$ & $291$ & $87.5$ & $15.6$ & $94.3$ & $88.5$ & $81.9$ & $69.1$ & $65.6$ & $37.5$ & $67.5$ \\
\rowcolor{lightyellow} SVDq+Sparsity & $1.25$ & $\textbf{410}$ & $86.5$ & $15.6$ & $93.5$ & $88.0$ & $83.7$ & $68.1$ & $62.5$ & $36.5$ & $66.8$ \\
\rowcolor{lightgreen} SVDq+Sparsity+V4 & $2.25$ & $227$ & $86.5$ & $20.8$ & $95.1$ & $89.6$ & $84.4$ & $70.1$ & $66.7$ & $39.6$ & $\underline{69.1}$ \\
\rowcolor{lightgreen} SVDq+Sparsity+V4 & $1.75$ & $291$ & $86.5$ & $18.8$ & $93.5$ & $90.4$ & $82.5$ & $68.1$ & $64.6$ & $36.5$ & $67.6$ \\
\rowcolor{lightgreen} SVDq+Sparsity+V4 & $1.25$ & $\textbf{410}$ & $86.5$ & $16.7$ & $92.4$ & $87.0$ & $83.3$ & $68.4$ & $62.5$ & $39.6$ & $67.0$ \\
\midrule
\multicolumn{11}{c}{\textbf{Qwen2.5-3B-Instruct}} \\ 
\midrule
\rowcolor{lightlightgray} Default & $16$ & $1$ & $78.1$ & $27.1$ & $89.8$ & $88.8$ & $81.0$ & $72.2$ & $41.7$ & $30.2$ & $63.6$ \\
ShadowKV SparsityOnly & $16$ & $32$ & $77.1$ & $18.8$ & $83.6$ & $81.8$ & $75.2$ & $48.6$ & $43.8$ & $28.1$ & $\underline{57.1}$ \\
ShadowKV & $2.5$ & $205$ & $75$ & $17.7$ & $69.3$ & $71.4$ & $69.2$ & $50.7$ & $32.3$ & $29.2$ & $51.8$ \\
\rowcolor{lightyellow} SVDq+Sparsity & $2.25$ & $227$ & $78.1$ & $19.8$ & $82.0$ & $83.6$ & $77.3$ & $47.2$ & $36.5$ & $28.1$ & $56.6$ \\
\rowcolor{lightyellow} SVDq+Sparsity & $1.75$ & $291$ & $80.2$ & $20.8$ & $80.7$ & $83.3$ & $76.9$ & $49.7$ & $38.5$ & $27.1$ & $\textbf{57.2}$ \\
\rowcolor{lightyellow} SVDq+Sparsity & $1.25$ & $\textbf{410}$ & $75.0$ & $17.7$ & $78.9$ & $82.6$ & $77.1$ & $46.9$ & $35.4$ & $30.2$ & $55.5$ \\
\rowcolor{lightgreen} SVDq+Sparsity+V4 & $2.25$ & $227$ & $75.0$ & $20.5$ & $80.2$ & $83.9$ & $78.7$ & $45.8$ & $38.5$ & $28.1$ & $56.4$ \\
\rowcolor{lightgreen} SVDq+Sparsity+V4 & $1.75$ & $291$ & $80.2$ & $19.8$ & $81.8$ & $83.3$ & $76.0$ & $49.0$ & $37.5$ & $29.2$ & $\underline{57.1}$ \\
\rowcolor{lightgreen} SVDq+Sparsity+V4 & $1.25$ & $\textbf{410}$ & $77.1$ & $13.5$ & $77.3$ & $81.2$ & $76.2$ & $46.9$ & $33.3$ & $29.2$ & $54.4$ \\
\bottomrule
\end{tabular}
}
\caption{\label{tab:SVDspresult}
Performance of our method in conjunction with the sparsity strategy of ShadowKV, denoted by "SVDq+Sparsity", in different models on the RULER benchmark evaluated at a context length of $64$K. The third column key compression ratio ("CR") is computed by $16 / \bar{b} \times$ the sparsity ratio, $32$, and represents the compression ratio of the key cache that involves in the attention computation. The second row ("ShadowKV SparsityOnly") refers to applying only the sparsity strategy of ShadowKV without any quantization or SVD low-rank methods. It acts as another baseline for comparison. For the third row ("ShadowKV"), in the Llama-3.1 model, we use the same settings as in the ShadowKV paper, retaining 160 ranks of the SVD and truncating the rest, which is equivalent to a quantization bit width of $2.5$. For the Qwen2.5-14B, 7B and 3B models, to maintain consistent quantization bit widths ($2.5$ bits), we retain $160$, $80$ and $40$ ranks, respectively. The quantization bit schedules for "SVDq+Sparsity" (in yellow) are identical for all four models and are shown in Table \ref{tab:bitschedule} in Appendix \ref{app:5.2}. In addition to the yellow rows, the rows "SVDq+Sparsity+V4" (in green) introduce an auxiliary $4$-bit per-token quantization in the $V$ cache. Our method outperforms ShadowKV despite having a lower (mixed) bit width. 
}
\end{table*}

\subsection{Results of SVDq with Sparsity and $V$ Quantization} \label{subsec:ResSVDqspQV}

In the final experiment, we repeat the second experiment while additionally introducing a quantization method to the $V$ cache to further reduce the required memory for model loading. Please find the experiment settings in Appendix \ref{app:5.3}. The results are presented in the green rows in Table \ref{tab:SVDspresult}.

Our observations indicate a very small performance loss compared to the yellow rows (without V cache quantization) in Table \ref{tab:SVDspresult}. This suggests that, despite being an approximation method with a very low compression rate, SVDq does not significantly degrade model performance even when combined with sparsity and $V$ cache compression.

The resulting insignificant performance degeneration while combining sparsity and V cache quantization not only demonstrate the effectiveness of the SVD quantization method in $K$ cache compression but also highlight its compatibility with existing compression techniques.

\subsection{Results of LongBench benchmark}\label{subsec:res_longbench}

We also implement numerical experiments based on the LongBench benchmark \cite{bai2023longbench} and exclude the tests of which the sequence lengths are less than $4$K. The baselines and configurations of our method are the same as those presented in Section \ref{sec:experiments}. The results are shown in Table \ref{tab:longbench}. Note that the second row for each model, which includes the results for "ShadowKV SparsityOnly," "ShadowKV", three "SVDq+Sparsity", and three "SVDq+Sparsity+4V" configurations, corresponds to the results in Table \ref{tab:SVDspresult}. For most of the models and method configurations, our SVDq method either outperforms or exhibits comparable performance to the baselines, including per-channel quantization \cite{liu2024kivi}, ThinK \cite{xu_think_2024}, and ShadowKV \cite{sun_shadowkv_2024}.  Notably, the performance degradation of our method compared to the full, non-quantized model is insignificant and nearly lossless for LongBench datasets. These results further corroborate the conclusions drawn from our analysis of the RULER benchmark.

\begin{table*}
\centering
\resizebox{\textwidth}{!}{
\begin{tabular}{crr|rrrrrr|r}
\toprule
Method & bit & CR & NarrativeQA & HotpotQA & MuSiQue & GovRepprt & SAMSum & RepoBench-P & Average \\ \midrule
\multicolumn{10}{c}{\textbf{Llama-3.1-8B-Instruct}} \\ 
\midrule
\rowcolor{lightlightgray} Default & $16$ & $1$ & $22.3$ & $17.5$ & $14.2$ & $33.4$ & $35.7$ & $43.4$ &  $30.3$ \\
Per-channel Quant & $3$ & $5.3$ & $17.7$ & $15.9$ & $6.15$ & $33.0$ & $35.4$ & $30.9$ & $24.1$ \\ 
ThinK & $3$ & $5.3$ & $14.0$ & $15.4$ & $11.0$ & $33.0$ & $35.2$ & $48.8$ & $\underline{30.0}$ \\ 
\rowcolor{lightmintbg} SVDq(ours) & $3$ & $5.3$ & $20.2$ & $16.3$ & $11.0$ & $34.2$ & $35.3$ & $45.1$ & $\underline{30.0}$ \\ 
\rowcolor{lightmintbg} SVDq(ours) & $2$ & $\textbf{8.0}$ & $18.4$ & $18.0$ & $11.5$ & $32.3$ & $34.7$ & $48.5$ & $\textbf{30.8}$ \\ 
\hline
ShadowKV SparsityOnly & $16$ & $32$ & $21.9$ & $20.8$ & $10.3$ & $33.0$ & $36.2$ & $44.2$ & $\underline{30.4}$  \\
ShadowKV & $2.5$ & $205$ & $22.6$ & $21.5$ & $10.7$ & $32.5$ & $37.1$ & $45.6$ & $\textbf{31.2}$ \\
\rowcolor{lightyellow} SVDq+Sparsity & $2.25$ & $227$ & $22.3$ & $21.4$ & $9.54$ & $33.2$ & $36.2$ & $42.3$ & $29.9$ \\
\rowcolor{lightyellow} SVDq+Sparsity & $1.75$ & $291$ & $22.8$ & $21.3$ & $10.3$ & $33.4$ & $35.2$ & $43.7$ & $\underline{30.4}$ \\
\rowcolor{lightyellow} SVDq+Sparsity & $1.25$ & $\textbf{410}$ & $20.8$ & $17.9$ & $11.1$ & $33.0$ & $34.2$ & $43.1$ & $29.4$ \\
\rowcolor{lightgreen} SVDq+Sparsity+V4 & $2.25$ & $227$ & $22.0$ & $19.6$ & $13.1$ & $33.6$ & $35.4$ & $41.3$ & $29.7$ \\ 
\rowcolor{lightgreen} SVDq+Sparsity+V4 & $1.75$ & $291$ & $22.3$ & $19.5$ & $11.4$ & $33.6$ & $34.9$ & $44.1$ & $30.3$ \\ 
\rowcolor{lightgreen} SVDq+Sparsity+V4 & $1.25$ & $\textbf{410}$ & $20.9$ & $22.1$ & $11.9$ & $33.1$ & $37.0$ & $41.8$ & $30.0$ \\ 
\midrule
\multicolumn{10}{c}{\textbf{Qwen2.5-14B-Instruct}} \\ 
\midrule
\rowcolor{lightlightgray} Default & $16$ & $1$ & $7.16$ & $17.0$ & $10.6$ & $30.5$ & $41.5$ & $44.7$ & $25.2$ \\
Per-channel Quant & $3$ & $5.3$ & $6.72$ & $13.8$ & $7.64$ & $30.7$ & $39.0$ & $43.5$ & $23.5$ \\ 
ThinK & $3$ & $5.3$ & $9.25$ & $12.8$ & $9.44$ & $30.2$ & $40.5$ & $44.3$ & $\underline{24.4}$ \\ 
\rowcolor{lightmintbg} SVDq(ours) & $3$ & $5.3$ & $10.1$ & $19.6$ & $11.1$ & $30.7$ & $43.2$ & $42.2$ & $\textbf{26.1}$ \\ 
\rowcolor{lightmintbg} SVDq(ours) & $2$ & $\textbf{8.0}$ & $6.83$ & $13.6$ & $8.37$ & $30.9$ & $40.8$ & $38.1$ & $23.1$ \\ 
\hline
ShadowKV SparsityOnly & $16$ & $32$ & $7.99$ & $18.7$ & $10.6$ & $30.6$ & $40.6$ & $44.6$ & $25.5$ \\
ShadowKV & $2.5$ & $205$ & $7.46$ & $16.8$ & $12.3$ & $30.6$ & $41.4$ & $45.2$ & $25.6$ \\
\rowcolor{lightyellow} SVDq+Sparsity & $2.25$ & $227$ & $8.23$ & $19.3$ & $11.1$ & $31.1$ & $42.7$ & $45.7$ & $\underline{26.4}$ \\
\rowcolor{lightyellow} SVDq+Sparsity & $1.75$ & $291$ & $10.2$ & $18.7$ & $12.0$ & $30.8$ & $40.2$ & $45.0$ & $26.2$ \\
\rowcolor{lightyellow} SVDq+Sparsity & $1.25$ & $\textbf{410}$ & $8.49$ & $21.0$ & $12.5$ & $30.3$ & $41.3$ & $46.6$ & $\textbf{26.7}$ \\
\rowcolor{lightgreen} SVDq+Sparsity+V4 & $2.25$ & $227$ & $7.11$ & $16.4$ & $12.2$ & $30.6$ & $42.1$ & $47.3$ & $25.9$ \\ 
\rowcolor{lightgreen} SVDq+Sparsity+V4 & $1.75$ & $291$ & $7.88$ & $18.8$ & $12.7$ & $30.9$ & $41.8$ & $45.5$ & $26.3$ \\ 
\rowcolor{lightgreen} SVDq+Sparsity+V4 & $1.25$ & $\textbf{410}$ & $7.33$ & $16.7$ & $13.1$ & $30.6$ & $40.8$ & $42.4$ & $25.2$ \\ 
\midrule
\multicolumn{10}{c}{\textbf{Qwen2.5-7B-Instruct}} \\ 
\midrule
\rowcolor{lightlightgray} Default & $16$ & $1$ & $8.78$ & $11.2$ & $7.35$ & $31.5$ & $40.1$ & $49.3$ & $28.7$ \\
Per-channel Quant & $3$ & $5.3$ & $6.46$ & $12.3$ & $5.69$ & $30.6$ & $41.1$ & $44.3$ & $26.6$ \\ 
ThinK & $3$ & $5.3$ & $9.02$ & $11.6$ & $6.15$ & $31.1$ & $38.3$ & $54.1$ & $\textbf{29.8}$  \\ 
\rowcolor{lightmintbg} SVDq(ours) & $3$ & $5.3$ & $8.80$ & $11.3$ & $8.32$ & $31.1$ & $40.2$ & $48.9$ & $28.6$ \\ 
\rowcolor{lightmintbg} SVDq(ours) & $2$ & $\textbf{8.0}$ & $6.84$ & $19.9$ & $9.47$ & $31.9$ & $40.4$ & $48.5$ & $\underline{29.7}$ \\ 
\hline
ShadowKV SparsityOnly & $16$ & $32$ & $10.5$ & $10.5$ & $7.78$ & $31.8$ & $38.9$ & $49.9$ & $29.0$ \\
ShadowKV & $2.5$ & $205$ & $10.3$ & $12.0$ & $8.06$ & $30.9$ & $40.1$ & $49.1$ & $29.0$ \\
\rowcolor{lightyellow} SVDq+Sparsity & $2.25$ & $227$ & $11.3$ & $11.2$ & $7.10$ & $31.4$ & $41.5$ & $50.6$ & $29.6$ \\
\rowcolor{lightyellow} SVDq+Sparsity & $1.75$ & $291$ & $10.3$ & $11.5$ & $7.14$ & $31.5$ & $39.7$ & $52.1$ & $\underline{29.7}$ \\
\rowcolor{lightyellow} SVDq+Sparsity & $1.25$ & $\textbf{410}$ & $9.74$ & $11.0$ & $7.74$ & $31.5$ & $40.7$ & $51.5$ & $29.6$ \\
\rowcolor{lightgreen} SVDq+Sparsity+V4 & $2.25$ & $227$ & $10.5$ & $11.2$ & $8.49$ & $31.6$ & $40.5$ & $51.4$ & $\textbf{29.8}$ \\ 
\rowcolor{lightgreen} SVDq+Sparsity+V4 & $1.75$ & $291$ & $7.83$ & $10.5$ & $7.83$ & $31.3$ & $40.1$ & $53.5$ & $\textbf{29.8}$ \\ 
\rowcolor{lightgreen} SVDq+Sparsity+V4 & $1.25$ & $\textbf{410}$ & $9.59$ & $10.8$ & $7.37$ & $31.0$ & $40.7$ & $52.5$ & $\textbf{29.8}$ \\ 
\midrule
\multicolumn{10}{c}{\textbf{Qwen2.5-3B-Instruct}} \\ 
\midrule
\rowcolor{lightlightgray} Default & $16$ & $1$ & $6.87$ & $14.4$ & $10.1$ & $30.6$ & $37.6$ & $46.1$ & $27.8$ \\
Per-channel Quant & $3$ & $5.3$ & $6.32$ & $9.47$ & $4.13$ & $29.2$ & $35.6$ & $44.6$ & $25.3$ \\ 
ThinK & $3$ & $5.3$ & $6.39$ & $8.11$ & $5.72$ & $29.8$ & $36.3$ & $43.9$ & $25.2$ \\ 
\rowcolor{lightmintbg} SVDq(ours) & $3$ & $5.3$ & $7.33$ & $14.5$ & $7.55$ & $29.9$ & $35.8$ & $48.2$ & $\textbf{27.9}$ \\ 
\rowcolor{lightmintbg} SVDq(ours) & $2$ & $\textbf{8.0}$ & $3.26$ & $8.06$ & $5.17$ & $26.1$ & $35.3$ & $53.0$ & $\underline{27.0}$ \\ 
\hline
ShadowKV SparsityOnly & $16$ & $32$ & $8.32$ & $14.2$ & $8.54$ & $29.8$ & $37.7$ & $50.0$ & $\textbf{28.9}$ \\
ShadowKV & $2.5$ & $205$ & $7.19$ & $15.8$ & $9.04$ & $27.4$ & $37.5$ & $47.2$ & $27.8$ \\
\rowcolor{lightyellow} SVDq+Sparsity & $2.25$ & $227$ & $7.14$ & $15.2$ & $9.76$ & $30.0$ & $38.3$ & $46.7$ & ${28.1}$ \\
\rowcolor{lightyellow} SVDq+Sparsity & $1.75$ & $291$ & $7.51$ & $15.0$ & $7.27$ & $29.3$ & $37.6$ & $46.6$ & $27.5$ \\
\rowcolor{lightyellow} SVDq+Sparsity & $1.25$ & $\textbf{410}$ & $8.15$ & $14.9$ & $8.09$ & $29.3$ & $38.4$ & $48.2$ & $\underline{28.4}$\\
\rowcolor{lightgreen} SVDq+Sparsity+V4 & $2.25$ & $227$ & $7.22$ & $14.0$ & $8.45$ & $29.0$ & $38.2$ & $47.3$ & $27.8$ \\ 
\rowcolor{lightgreen} SVDq+Sparsity+V4 & $1.75$ & $291$ & $6.41$ & $14.3$ & $10.3$ & $29.2$ & $35.4$ & $48.1$ & $27.9$ \\ 
\rowcolor{lightgreen} SVDq+Sparsity+V4 & $1.25$ & $\textbf{410}$ & $7.68$ & $13.9$ & $8.26$ & $29.1$ & $37.4$ & $46.8$ & $27.6$ \\ 
\bottomrule
\end{tabular}
}
\caption{Results of the LongBench benchmarks \cite{bai2023longbench} (longer than $4$K). The experiment settings are the same as those for RULER benchmarks in Section \ref{sec:experiments}. The second row for each model, which includes the results for "ShadowKV SparsityOnly," "ShadowKV", three "SVDq+Sparsity", and three "SVDq+Sparsity+V4" configurations, corresponds to the results in Table \ref{tab:SVDspresult}. }
\label{tab:longbench}
\end{table*}

\section{Conclusions}\label{sec:conclusions}

We present a mixed precision quantization approach for \(KV\) cache compression, which is grounded in projection representation within the SVD and singular vector space. In this method, we assign higher quantization bit widths to the initial latent channels and gradually reduce the bit widths for subsequent latent channels. Additionally, there is an option to truncate the final channels. Through comprehensive experiments, we show that this approach outperforms direct per - channel quantization in terms of model performance, even when using lower mixed bit widths.
Moreover, we explore the performance of our proposed method when integrated with other \(KV\) cache compression techniques, such as sparsity and \(V\) cache quantization. Our results reveal that our method incurs minimal performance degradation, even when extremely low equivalent quantization bit widths (mixed \(1.75\) and \(1.25\) bits for the \(K\) cache) are utilized. Overall, these findings convincingly demonstrate the effectiveness and efficiency of our proposed method in \(K\) cache compression.

\section{Limitations}\label{sec:limitations}

Although our method demonstrates good effectiveness in $K$ cache compression, it primarily reduces the required memory space for model loading without directly addressing computational cost. In fact, our current implementation may even slightly increase inference time. 

Specifically, we utilize the pre-RoPE setting in our implementation. Our method extracts a quantized low-bit $K$ cache of the SVD projection representation before the application of Rotary Position Embeddings (RoPE) and shares this low-bit representation across all heads. Due to the online computation of RoPE, which depends on the incoming position index, the reconstruction from the projection representation to the original representation cannot be efficiently integrated into the model's forward pass. Consequently, this leads to an increase in computational cost for each head.

This increase in computational cost could potentially be remedied by switching to the post-RoPE setting, where $K$ cache is handled after the application of RoPE. However, as reported in ShadowKV work \citep{sun_shadowkv_2024} and observed in our numerical tests, the post-RoPE setting generally exhibits degraded performance compared to the pre-RoPE setting. 

Therefore, investigating methods to accelerate the computation of our SVD quantization method, potentially by exploring alternative approaches or optimizations within the pre-RoPE framework, is an interesting direction for future research.

\bibliographystyle{neurips_2024}
\bibliography{SVDq}

\begin{thebibliography}{39}
\providecommand{\natexlab}[1]{#1}
\providecommand{\url}[1]{\texttt{#1}}
\expandafter\ifx\csname urlstyle\endcsname\relax
  \providecommand{\doi}[1]{doi: #1}\else
  \providecommand{\doi}{doi: \begingroup \urlstyle{rm}\Url}\fi

\bibitem[Ainslie et~al.(2023)Ainslie, Lee-Thorp, Jong, Zemlyanskiy, Lebrón, and Sanghai]{ainslie_gqa_2023}
Ainslie, J., Lee-Thorp, J., Jong, M.~d., Zemlyanskiy, Y., Lebrón, F., and Sanghai, S.
\newblock {GQA}: {Training} {Generalized} {Multi}-{Query} {Transformer} {Models} from {Multi}-{Head} {Checkpoints}, December 2023.
\newblock URL \url{http://arxiv.org/abs/2305.13245}.
\newblock arXiv:2305.13245 [cs].

\bibitem[Bai et~al.(2023)Bai, Lv, Zhang, Lyu, Tang, Huang, Du, Liu, Zeng, Hou, Dong, Tang, and Li]{bai2023longbench}
Bai, Y., Lv, X., Zhang, J., Lyu, H., Tang, J., Huang, Z., Du, Z., Liu, X., Zeng, A., Hou, L., Dong, Y., Tang, J., and Li, J.
\newblock Longbench: A bilingual, multitask benchmark for long context understanding.
\newblock \emph{arXiv preprint arXiv:2308.14508}, 2023.

\bibitem[Chang et~al.(2024)Chang, Lin, Lin, Chen, Hu, Wang, Huang, Ceze, Abdelfattah, and Wu]{chang_palu_2024}
Chang, C.-C., Lin, W.-C., Lin, C.-Y., Chen, C.-Y., Hu, Y.-F., Wang, P.-S., Huang, N.-C., Ceze, L., Abdelfattah, M.~S., and Wu, K.-C.
\newblock Palu: {Compressing} {KV}-{Cache} with {Low}-{Rank} {Projection}, November 2024.
\newblock URL \url{http://arxiv.org/abs/2407.21118}.
\newblock arXiv:2407.21118 [cs].

\bibitem[Chen et~al.(2024)Chen, Zhang, Gao, Mullins, Constantinides, and Zhao]{chen_optimised_2024}
Chen, Y., Zhang, C., Gao, X., Mullins, R.~D., Constantinides, G.~A., and Zhao, Y.
\newblock Optimised {Grouped}-{Query} {Attention} {Mechanism} for {Transformers}, June 2024.
\newblock URL \url{http://arxiv.org/abs/2406.14963}.
\newblock arXiv:2406.14963 [cs].

\bibitem[{DeepSeek-AI} et~al.(2025){DeepSeek-AI}, Guo, Yang, Zhang, Song, Zhang, Xu, Zhu, Ma, Wang, et~al.]{deepseek-ai_deepseek-r1_2025}
{DeepSeek-AI}, Guo, D., Yang, D., Zhang, H., Song, J., Zhang, R., Xu, R., Zhu, Q., Ma, S., Wang, P., et~al.
\newblock {DeepSeek}-{R1}: {Incentivizing} {Reasoning} {Capability} in {LLMs} via {Reinforcement} {Learning}, 2025.
\newblock URL \url{https://arxiv.org/abs/2501.12948}.
\newblock Version Number: 1.

\bibitem[Ge et~al.(2023)Ge, Zhang, Liu, Zhang, Han, and Gao]{ge2023fastgen}
Ge, S., Zhang, Y., Liu, L., Zhang, M., Han, J., and Gao, J.
\newblock Model tells you what to discard: Adaptive kv cache compression for llms.
\newblock \emph{arXiv preprint arXiv:2310.01801}, 2023.

\bibitem[Grattafiori et~al.(2024)Grattafiori, Dubey, Jauhri, Pandey, Kadian, Al-Dahle, Letman, Mathur, Schelten, Vaughan, et~al.]{grattafiori_llama_2024}
Grattafiori, A., Dubey, A., Jauhri, A., Pandey, A., Kadian, A., Al-Dahle, A., Letman, A., Mathur, A., Schelten, A., Vaughan, A., et~al.
\newblock The {Llama} 3 {Herd} of {Models}, November 2024.
\newblock URL \url{http://arxiv.org/abs/2407.21783}.
\newblock arXiv:2407.21783 [cs].

\bibitem[Hooper et~al.(2024)Hooper, Kim, Mohammadzadeh, Mahoney, Shao, Keutzer, and Gholami]{hooper2024kvquant}
Hooper, C., Kim, S., Mohammadzadeh, H., Mahoney, M.~W., Shao, Y.~S., Keutzer, K., and Gholami, A.
\newblock Kvquant: Towards 10 million context length llm inference with kv cache quantization.
\newblock \emph{arXiv preprint arXiv:2401.18079}, 2024.

\bibitem[Hsieh et~al.(2024)Hsieh, Sun, Kriman, Acharya, Rekesh, Jia, Zhang, and Ginsburg]{hsieh_ruler_2024}
Hsieh, C.-P., Sun, S., Kriman, S., Acharya, S., Rekesh, D., Jia, F., Zhang, Y., and Ginsburg, B.
\newblock {RULER}: {What}'s the {Real} {Context} {Size} of {Your} {Long}-{Context} {Language} {Models}?, August 2024.
\newblock URL \url{http://arxiv.org/abs/2404.06654}.
\newblock arXiv:2404.06654 [cs].

\bibitem[Jin et~al.(2024)Jin, Song, Zhou, and Qin]{jin_align_2024}
Jin, Q., Song, X., Zhou, F., and Qin, Z.
\newblock Align {Attention} {Heads} {Before} {Merging} {Them}: {An} {Effective} {Way} for {Converting} {MHA} to {GQA}, December 2024.
\newblock URL \url{http://arxiv.org/abs/2412.20677}.
\newblock arXiv:2412.20677 [cs].

\bibitem[Li et~al.(2024{\natexlab{a}})Li, Lin, Zhang, Cai, Li, Guo, Xie, Meng, Zhu, and Han]{li_svdquant_2024}
Li, M., Lin, Y., Zhang, Z., Cai, T., Li, X., Guo, J., Xie, E., Meng, C., Zhu, J.-Y., and Han, S.
\newblock {SVDQuant}: {Absorbing} {Outliers} by {Low}-{Rank} {Components} for 4-{Bit} {Diffusion} {Models}, November 2024{\natexlab{a}}.
\newblock URL \url{http://arxiv.org/abs/2411.05007}.
\newblock arXiv:2411.05007 [cs].

\bibitem[Li et~al.(2025)Li, Xing, Li, Qu, Zhen, Liu, Yao, Pan, and Yuan]{li2025kvtuner}
Li, X., Xing, Z., Li, Y., Qu, L., Zhen, H.-L., Liu, W., Yao, Y., Pan, S.~J., and Yuan, M.
\newblock Kvtuner: Sensitivity-aware layer-wise mixed precision kv cache quantization for efficient and nearly lossless llm inference, 2025.
\newblock URL \url{https://arxiv.org/abs/2502.04420}.

\bibitem[Li et~al.(2024{\natexlab{b}})Li, Huang, Yang, Venkitesh, Locatelli, Ye, Cai, Lewis, and Chen]{li2024snapkv}
Li, Y., Huang, Y., Yang, B., Venkitesh, B., Locatelli, A., Ye, H., Cai, T., Lewis, P., and Chen, D.
\newblock Snapkv: Llm knows what you are looking for before generation.
\newblock \emph{arXiv preprint arXiv:2404.14469}, 2024{\natexlab{b}}.

\bibitem[Lin et~al.(2024)Lin, Tang, Yang, Zhang, Xiao, Gan, and Han]{lin_qserve_2024}
Lin, Y., Tang, H., Yang, S., Zhang, Z., Xiao, G., Gan, C., and Han, S.
\newblock {QServe}: {W4A8KV4} {Quantization} and {System} {Co}-design for {Efficient} {LLM} {Serving}, May 2024.
\newblock URL \url{http://arxiv.org/abs/2405.04532}.
\newblock arXiv:2405.04532 [cs].

\bibitem[Liu et~al.(2024{\natexlab{a}})Liu, Feng, Xue, Wang, Wu, Lu, Zhao, Deng, Zhang, Ruan, et~al.]{liu2024deepseekv3}
Liu, A., Feng, B., Xue, B., Wang, B., Wu, B., Lu, C., Zhao, C., Deng, C., Zhang, C., Ruan, C., et~al.
\newblock Deepseek-v3 technical report.
\newblock \emph{arXiv preprint arXiv:2412.19437}, 2024{\natexlab{a}}.

\bibitem[Liu et~al.(2024{\natexlab{b}})Liu, Gao, Zhao, Ma, Wang, and Wen]{liu_unlocking_2024}
Liu, P., Gao, Z.-F., Zhao, W.~X., Ma, Y., Wang, T., and Wen, J.-R.
\newblock Unlocking {Data}-free {Low}-bit {Quantization} with {Matrix} {Decomposition} for {KV} {Cache} {Compression}, May 2024{\natexlab{b}}.
\newblock URL \url{http://arxiv.org/abs/2405.12591}.
\newblock arXiv:2405.12591 [cs].

\bibitem[Liu et~al.(2023)Liu, Yuan, Jin, Zhong, Xu, Braverman, Chen, and Hu]{liu_kivi_2023}
Liu, Z., Yuan, J., Jin, H., Zhong, S., Xu, Z., Braverman, V., Chen, B., and Hu, X.
\newblock {KIVI}: {A} {Tuning}-{Free} {Asymmetric} 2bit {Quantization} for {KV} {Cache}, 2023.
\newblock URL \url{http://arxiv.org/abs/2402.02750}.
\newblock arXiv:2402.02750 [cs].

\bibitem[Liu et~al.(2024{\natexlab{c}})Liu, Yuan, Jin, Zhong, Xu, Braverman, Chen, and Hu]{liu2024kivi}
Liu, Z., Yuan, J., Jin, H., Zhong, S., Xu, Z., Braverman, V., Chen, B., and Hu, X.
\newblock Kivi: A tuning-free asymmetric 2bit quantization for kv cache.
\newblock \emph{arXiv preprint arXiv:2402.02750}, 2024{\natexlab{c}}.

\bibitem[Mirsky(1960)]{mirsky_symmetric_1960}
Mirsky, L.
\newblock {SYMMETRIC} {GAUGE} {FUNCTIONS} {AND} {UNITARILY} {INVARIANT} {NORMS}.
\newblock \emph{The Quarterly Journal of Mathematics}, 11\penalty0 (1):\penalty0 50--59, 1960.
\newblock ISSN 0033-5606, 1464-3847.
\newblock \doi{10.1093/qmath/11.1.50}.
\newblock URL \url{https://academic.oup.com/qjmath/article-lookup/doi/10.1093/qmath/11.1.50}.

\bibitem[OpenAI et~al.(2024)OpenAI, Achiam, Adler, Agarwal, Ahmad, Akkaya, Aleman, Almeida, Altenschmidt, Altman, et~al.]{openai_gpt-4_2024}
OpenAI, Achiam, J., Adler, S., Agarwal, S., Ahmad, L., Akkaya, I., Aleman, F.~L., Almeida, D., Altenschmidt, J., Altman, S., et~al.
\newblock {GPT}-4 {Technical} {Report}, March 2024.
\newblock URL \url{http://arxiv.org/abs/2303.08774}.
\newblock arXiv:2303.08774 [cs].

\bibitem[Ping et~al.(2024)Ping, Wang, Wang, Han, Xu, Yan, Chen, Chang, Liu, and Sun]{ping_delta-come_2024}
Ping, B., Wang, S., Wang, H., Han, X., Xu, Y., Yan, Y., Chen, Y., Chang, B., Liu, Z., and Sun, M.
\newblock Delta-{CoMe}: {Training}-{Free} {Delta}-{Compression} with {Mixed}-{Precision} for {Large} {Language} {Models}, November 2024.
\newblock URL \url{http://arxiv.org/abs/2406.08903}.
\newblock arXiv:2406.08903 [cs].

\bibitem[Pope et~al.(2022)Pope, Douglas, Chowdhery, Devlin, Bradbury, Levskaya, Heek, Xiao, Agrawal, and Dean]{pope_efficiently_2022}
Pope, R., Douglas, S., Chowdhery, A., Devlin, J., Bradbury, J., Levskaya, A., Heek, J., Xiao, K., Agrawal, S., and Dean, J.
\newblock Efficiently {Scaling} {Transformer} {Inference}, November 2022.
\newblock URL \url{http://arxiv.org/abs/2211.05102}.
\newblock arXiv:2211.05102 [cs].

\bibitem[Qwen et~al.(2025)Qwen, Yang, Yang, Zhang, Hui, Zheng, Yu, Li, Liu, Huang, Wei, Lin, Yang, Tu, Zhang, Yang, Yang, Zhou, Lin, Dang, Lu, Bao, Yang, Yu, Li, Xue, Zhang, Zhu, Men, Lin, Li, Tang, Xia, Ren, Ren, Fan, Su, Zhang, Wan, Liu, Cui, Zhang, and Qiu]{qwen_qwen25_2025}
Qwen, Yang, A., Yang, B., Zhang, B., Hui, B., Zheng, B., Yu, B., Li, C., Liu, D., Huang, F., Wei, H., Lin, H., Yang, J., Tu, J., Zhang, J., Yang, J., Yang, J., Zhou, J., Lin, J., Dang, K., Lu, K., Bao, K., Yang, K., Yu, L., Li, M., Xue, M., Zhang, P., Zhu, Q., Men, R., Lin, R., Li, T., Tang, T., Xia, T., Ren, X., Ren, X., Fan, Y., Su, Y., Zhang, Y., Wan, Y., Liu, Y., Cui, Z., Zhang, Z., and Qiu, Z.
\newblock Qwen2.5 {Technical} {Report}, January 2025.
\newblock URL \url{http://arxiv.org/abs/2412.15115}.
\newblock arXiv:2412.15115 [cs].

\bibitem[Ribar et~al.(2024)Ribar, Chelombiev, Hudlass-Galley, Blake, Luschi, and Orr]{ribar_sparq_2024}
Ribar, L., Chelombiev, I., Hudlass-Galley, L., Blake, C., Luschi, C., and Orr, D.
\newblock {SparQ} {Attention}: {Bandwidth}-{Efficient} {LLM} {Inference}, September 2024.
\newblock URL \url{http://arxiv.org/abs/2312.04985}.
\newblock arXiv:2312.04985 [cs].

\bibitem[Saxena et~al.(2024)Saxena, Saha, Choudhary, and Roy]{saxena_eigen_2024}
Saxena, U., Saha, G., Choudhary, S., and Roy, K.
\newblock Eigen {Attention}: {Attention} in {Low}-{Rank} {Space} for {KV} {Cache} {Compression}, November 2024.
\newblock URL \url{http://arxiv.org/abs/2408.05646}.
\newblock arXiv:2408.05646 [cs].

\bibitem[Singhania et~al.(2024)Singhania, Singh, He, Feizi, and Bhatele]{singhania_loki_2024}
Singhania, P., Singh, S., He, S., Feizi, S., and Bhatele, A.
\newblock Loki: {Low}-rank {Keys} for {Efficient} {Sparse} {Attention}, November 2024.
\newblock URL \url{http://arxiv.org/abs/2406.02542}.
\newblock arXiv:2406.02542 [cs].

\bibitem[Sun et~al.(2024)Sun, Chang, Bao, Zheng, Zheng, Liu, Dong, Chi, and Chen]{sun_shadowkv_2024}
Sun, H., Chang, L.-W., Bao, W., Zheng, S., Zheng, N., Liu, X., Dong, H., Chi, Y., and Chen, B.
\newblock {ShadowKV}: {KV} {Cache} in {Shadows} for {High}-{Throughput} {Long}-{Context} {LLM} {Inference}, October 2024.
\newblock URL \url{http://arxiv.org/abs/2410.21465}.
\newblock arXiv:2410.21465 [cs].

\bibitem[Tan et~al.(2024)Tan, Wang, Yan, and Deng]{tan_alignedkv_2024}
Tan, Y., Wang, H., Yan, C., and Deng, Y.
\newblock {AlignedKV}: {Reducing} {Memory} {Access} of {KV}-{Cache} with {Precision}-{Aligned} {Quantization}, October 2024.
\newblock URL \url{http://arxiv.org/abs/2409.16546}.
\newblock arXiv:2409.16546 [cs].

\bibitem[Tang et~al.(2024)Tang, Zhao, Zhu, Xiao, Kasikci, and Han]{tang_quest_2024}
Tang, J., Zhao, Y., Zhu, K., Xiao, G., Kasikci, B., and Han, S.
\newblock Quest: {Query}-{Aware} {Sparsity} for {Efficient} {Long}-{Context} {LLM} {Inference}, August 2024.
\newblock URL \url{http://arxiv.org/abs/2406.10774}.
\newblock arXiv:2406.10774 [cs].

\bibitem[Vaswani et~al.(2023)Vaswani, Shazeer, Parmar, Uszkoreit, Jones, Gomez, Kaiser, and Polosukhin]{vaswani_attention_2023}
Vaswani, A., Shazeer, N., Parmar, N., Uszkoreit, J., Jones, L., Gomez, A.~N., Kaiser, L., and Polosukhin, I.
\newblock Attention {Is} {All} {You} {Need}, August 2023.
\newblock URL \url{http://arxiv.org/abs/1706.03762}.
\newblock arXiv:1706.03762 [cs].

\bibitem[Wang et~al.(2024{\natexlab{a}})Wang, Wang, Wang, Zhang, Zhou, and Qiu]{wang_bitstack_2024}
Wang, X., Wang, P., Wang, B., Zhang, D., Zhou, Y., and Qiu, X.
\newblock {BitStack}: {Fine}-{Grained} {Size} {Control} for {Compressed} {Large} {Language} {Models} in {Variable} {Memory} {Environments}, October 2024{\natexlab{a}}.
\newblock URL \url{http://arxiv.org/abs/2410.23918}.
\newblock arXiv:2410.23918 [cs].

\bibitem[Wang et~al.(2024{\natexlab{b}})Wang, Zheng, Wan, and Zhang]{wang_svd-llm_2024}
Wang, X., Zheng, Y., Wan, Z., and Zhang, M.
\newblock {SVD}-{LLM}: {Truncation}-aware {Singular} {Value} {Decomposition} for {Large} {Language} {Model} {Compression}, May 2024{\natexlab{b}}.
\newblock URL \url{http://arxiv.org/abs/2403.07378}.
\newblock arXiv:2403.07378 [cs].

\bibitem[Xiao et~al.(2023)Xiao, Tian, Chen, Han, and Lewis]{xiao2023streamingllm}
Xiao, G., Tian, Y., Chen, B., Han, S., and Lewis, M.
\newblock Efficient streaming language models with attention sinks.
\newblock \emph{arXiv preprint arXiv:2309.17453}, 2023.

\bibitem[Xu et~al.(2024)Xu, Jie, Dong, Wang, Lu, Zhou, Saha, Xiong, and Sahoo]{xu_think_2024}
Xu, Y., Jie, Z., Dong, H., Wang, L., Lu, X., Zhou, A., Saha, A., Xiong, C., and Sahoo, D.
\newblock {ThinK}: {Thinner} {Key} {Cache} by {Query}-{Driven} {Pruning}, October 2024.
\newblock URL \url{http://arxiv.org/abs/2407.21018}.
\newblock arXiv:2407.21018 [cs].

\bibitem[Yang et~al.(2024)Yang, Kim, Bae, Kwon, Park, Yang, Kwon, and Lee]{yang2024mikv}
Yang, J.~Y., Kim, B., Bae, J., Kwon, B., Park, G., Yang, E., Kwon, S.~J., and Lee, D.
\newblock No token left behind: Reliable kv cache compression via importance-aware mixed precision quantization.
\newblock \emph{arXiv preprint arXiv:2402.18096}, 2024.

\bibitem[Yang et~al.(2025)Yang, Wang, Li, Wang, Chen, Chen, Yu, Liu, Hao, Yuan, et~al.]{yang2025attentionpredictor}
Yang, Q., Wang, J., Li, X., Wang, Z., Chen, C., Chen, L., Yu, X., Liu, W., Hao, J., Yuan, M., et~al.
\newblock Attentionpredictor: Temporal pattern matters for efficient llm inference.
\newblock \emph{arXiv preprint arXiv:2502.04077}, 2025.

\bibitem[Zhang et~al.(2024{\natexlab{a}})Zhang, Ji, Chen, Fu, Miao, Nie, Chen, and Cui]{zhang2024pqcache}
Zhang, H., Ji, X., Chen, Y., Fu, F., Miao, X., Nie, X., Chen, W., and Cui, B.
\newblock Pqcache: Product quantization-based kvcache for long context llm inference.
\newblock \emph{arXiv preprint arXiv:2407.12820}, 2024{\natexlab{a}}.

\bibitem[Zhang et~al.(2024{\natexlab{b}})Zhang, Wang, Liu, Wang, Cheng, Zhang, and Shen]{zhang_lorc_2024}
Zhang, R., Wang, K., Liu, L., Wang, S., Cheng, H., Zhang, C., and Shen, Y.
\newblock {LoRC}: {Low}-{Rank} {Compression} for {LLMs} {KV} {Cache} with a {Progressive} {Compression} {Strategy}, October 2024{\natexlab{b}}.
\newblock URL \url{http://arxiv.org/abs/2410.03111}.
\newblock arXiv:2410.03111 [cs].

\bibitem[Zhang et~al.(2024{\natexlab{c}})Zhang, Sheng, Zhou, Chen, Zheng, Cai, Song, Tian, R{\'e}, Barrett, et~al.]{zhang2024h2o}
Zhang, Z., Sheng, Y., Zhou, T., Chen, T., Zheng, L., Cai, R., Song, Z., Tian, Y., R{\'e}, C., Barrett, C., et~al.
\newblock H2o: Heavy-hitter oracle for efficient generative inference of large language models.
\newblock \emph{Advances in Neural Information Processing Systems}, 36, 2024{\natexlab{c}}.

\end{thebibliography}

\newpage
\appendix

\section{Experiments Descriptions}
\label{sec:appendix}

\subsection{Descriptions for Section \ref{subsec:ResSVDq}}\label{app:5.1}

In this experiment, we include the below baselines for comparison: 

\noindent \textbf{Default} No compression is applied, and $16$-bit widths are used for all values. This is the default configuration of each models; 

\noindent \textbf{Direct $3$-bit Quantization} $3$-bit per-channel quantization \cite{liu_kivi_2023} is applied directly to the $\textbf{K}$ matrix in its original space (as depicted in Figure \ref{fig:origvar}). 

\noindent \textbf{ThinK} Direct channel truncation in the original space by ThinK \citep{xu_think_2024} that retains $\frac{3}{4}$ channels, in conjunction with $4$-bit quantization, results in an equivalent $3$-bit setting. 

\noindent The equivalent mixed quantization bit width in this experiment are selected as $\bar{b} = 3, 2$ for the SVDq method. The quantization schedule $\textbf{\textit{b}}$ is set to $(8, 4, 4, 4, 2, 2, 0, 0)$ and $(8, 4, 4, 0, 0, 0, 0, 0)$, respectively. 

\subsection{Descriptions for Section \ref{subsec:ResSVDqsp}}\label{app:5.2}

ShadowKV \cite{sun_shadowkv_2024} and its sparsity techniques act as baselines and utilized in this work. Briefly, this strategy divides the $K$ cache in the prefilling stage into small chunks, each containing $8$ tokens. It computes the mean embedding of each trunk as the landmark and then uses these landmarks to identify important chunks. Specifically, the top-$k$ chunks with the highest attention scores are considered important and retained, while the remaining chunks are neglected in the computation of attention. Note that this method also includes an auxiliary selection mechanism for outlier chunks, which are identified based on low cosine similarity. These outliers are not clipped during the sparsity process. In addition to sparsity, the full ShadowKV method incorporates SVD low-rank approximation of the $K$ cache, retaining $160$ out of the full $1024$ ranks. This low-rank approximation can be considered equivalent to approximately $2.5$-bit quantization, as the default numerical precision is $16$ bits.

Based on ShadowKV, the baseline results for comparison that shown in Table \ref{tab:SVDspresult} (the first three rows of each model) are:

\noindent \textbf{Default} Scores obtained with the default $16$-bit digital precision;

\noindent \textbf{Sparsity} Scores obtained using the ShadowKV sparsity method without low-rank approximation or quantization; 
 
\noindent \textbf{ShadowKV} Scores obtained using the full ShadowKV method, including both sparsity and equivalent $2.5$-bit quantization.

The detailed quantization schedules are shown in Table \ref{tab:bitschedule}.

\begin{table}[h]
\centering
\begin{tabular}{cc}
\hline
Equivalent bit $\bar{b}$ & schedule $\textit{\textbf{b}}$ \\
\hline
$2.25$ & $(8, 4, 4, 2, 0, 0, 0, 0)$ \\
$1.75$ & $(8, 4, 2, 0, 0, 0, 0, 0)$ \\
$1.25$ & $(4, 4, 2, 0, 0, 0, 0, 0)$ \\
\hline
\end{tabular}
\caption{\label{tab:bitschedule}
Key quantization bit schedules for SVDq.
}
\end{table}

\subsection{Descriptions and Results for Section \ref{subsec:ResSVDqspQV}}\label{app:5.3}

In this experiment, the configuration of $K$ cache compression and sparsity remains the same as in the second experiment: the mixed quantization bit schedules are set according to Table \ref{tab:bitschedule}, consistent with the previous experiment, and the sparsity method employs the ShadowKV sparsity technique \citep{sun_shadowkv_2024}. In addition to these settings, we observe the very weak low-rank property of $V$ cache and hence apply a direct $4$-bit per-token quantization to the $V$ cache.

\end{document}